\documentclass[11pt,a4paper]{article}

\usepackage[T1]{fontenc}
\usepackage[utf8]{inputenc}
\usepackage{lmodern}
\usepackage[margin=1in]{geometry}
\usepackage{microtype}

\usepackage{amsmath,amssymb,amsthm,mathtools}
\usepackage{bm}
\usepackage{booktabs}
\usepackage{graphicx}
\usepackage{url}
\usepackage[hidelinks]{hyperref}
\usepackage{orcidlink}

\usepackage[authoryear,round]{natbib}

\makeatletter
\let\latexincludegraphics\includegraphics
\renewcommand{\includegraphics}[2][]{%
	\IfFileExists{#2}{%
		\latexincludegraphics[#1]{#2}%
	}{%
		\fbox{%
			\begin{minipage}{0.9\linewidth}
				\centering
				\footnotesize Missing figure file: \texttt{\detokenize{#2}}
			\end{minipage}%
		}%
	}%
}
\makeatother

\theoremstyle{definition}
\newtheorem{definition}{Definition}[section]

\theoremstyle{plain}
\newtheorem{proposition}[definition]{Proposition}

\title{Structural interpretability in SVMs with truncated orthogonal polynomial kernels}

\author{%
V\'ictor Soto-Larrosa\,\orcidlink{0000-0002-7079-3646}\thanks{Department of Biosciences, Faculty of Biomedical and Health Sciences, Universidad Europea de Madrid, C/ Tajo s/n, 28670 Villaviciosa de Od\'on, Madrid, Spain. ORCID: 0000-0002-7079-3646. Emails: \href{mailto:v.soto@uah.es}{\nolinkurl{v.soto@uah.es}}; \href{mailto:victor.soto@universidadeuropea.es}{\nolinkurl{victor.soto@universidadeuropea.es}}.}%
\and
Nuria Torrado\,\orcidlink{0000-0002-1156-3613}\thanks{Corresponding author. Department of Mathematics, Faculty of Sciences, Universidad Aut\'onoma de Madrid, C/ Francisco Tom\'as y Valiente 7, 28049 Madrid, Spain; Instituto de Ciencias Matem\'aticas (ICMAT), C/ Nicol\'as Cabrera 13--15, 28049 Madrid, Spain. ORCID: 0000-0002-1156-3613. Emails: \href{mailto:nuria.torrado@uam.es}{\nolinkurl{nuria.torrado@uam.es}}; \href{mailto:nuria.torrado@icmat.es}{\nolinkurl{nuria.torrado@icmat.es}}.}%
\and
Edmundo J. Huertas\,\orcidlink{0000-0001-6802-3303}\thanks{Centro de Estudios Matem\'aticos, Facultad de Ciencias de la Empresa y la Tecnolog\'ia, Universidad a Distancia de Madrid (UDIMA), Ctra. de La Coru\~na, km 38.500, v\'ia servicio n\textordmasculine\ 15, 28400 Collado Villalba, Madrid, Spain. ORCID: 0000-0001-6802-3303. Email: \href{mailto:edmundojose.huertas@udima.es}{\nolinkurl{edmundojose.huertas@udima.es}}.}%
}

\date{}

\begin{document}

	\maketitle

\begin{center}
\today{}
\end{center}

\begin{abstract}
	Support vector machines (SVMs) can yield classifiers with similar predictive performance but different internal structures, particularly in the contributions of individual explanatory variables and their interactions. We introduce Orthogonal Representation Contribution Analysis (ORCA), a post-training methodology for SVMs constructed from tensor products of truncated Christoffel--Darboux kernels. The associated reproducing kernel Hilbert space (RKHS) admits a finite orthonormal polynomial basis in which the regularized component of the fitted decision function has an exact expansion. Squared expansion coefficients induce an additive decomposition of its squared RKHS norm. Based on this representation, normalized Orthogonal Kernel Contribution (OKC) indices quantify the fractions of the squared RKHS norm associated with total polynomial degrees, interaction orders, and input-variable contributions, both individually and through interactions of arbitrary order. ORCA summarizes these indices through joint structural maps, variable-specific displays, and comparisons of fitted classifiers. A controlled quartet of classification problems illustrates substantial structural differences under matched predictive performance. A double-spiral experiment compares five Jacobi kernels, while an application to the Vertebral Column data illustrates the method after cross-validated model selection. ORCA provides a quantitative characterization of the internal structure of fitted SVMs, revealing differences that may remain undetected by standard predictive performance metrics.
\end{abstract}

\medskip
	
	

\section{Introduction}
\label{sec:introduction}

A fundamental problem in binary classification is the construction of a decision rule capable of discriminating between two populations from a sample of explanatory variables. Among the nonparametric methodologies developed for this purpose, support vector machines (SVMs) provide a foundational and well-established regularized framework \citep{CortesVapnik1995,Vapnik1998,ScholkopfSmola2002}. By embedding the explanatory variables into a reproducing kernel Hilbert space (RKHS), kernelized SVMs can model nonlinear decision boundaries while controlling functional complexity via norm penalization. In practice, the assessment and comparison of fitted SVMs have traditionally focused on predictive performance and computational efficiency \citep{HornEtAl2018}. However, relying solely on predictive accuracy is often insufficient when structural interpretability is required. In particular, classifiers achieving nearly identical test accuracy may rely on markedly different internal mechanisms, differing in their polynomial complexity or in whether the decision rule is driven by pure marginal effects or by higher-order covariate interactions. While characterizing the internal structure of nonlinear classification rules has been addressed in classical parametric settings, for instance, \cite{BerrenderoCarcamo2019} factorize quadratic classifiers into products of hyperplanes to isolate discriminating linear components and identify relevant predictors, the kernel representation in SVMs acts largely as a black box with respect to the original coordinates, making the explicit roles of individual explanatory variables and their interactions difficult to quantify.

A natural approach to combine flexible kernel representations with an explicit functional structure is through the use of orthogonal polynomials.
Over the past two decades, several families of orthogonal polynomials have been employed to construct kernels for SVM classification, including multivariate kernels obtained as tensor products of univariate orthogonal-polynomial kernels acting on individual input variables. Early contributions introduced Chebyshev and Legendre kernels \citep{YeEtAl2006,OzerChenCirpan2011,PanChenYou2012}, followed by developments based on Hermite and Gegenbauer polynomials as well as modified orthogonal-polynomial constructions \citep{MoghaddamHamidzadeh2016,TianWang2017,PadiernaEtAl2018}, and more recent extensions to fractional orthogonal-polynomial families \citep{AmaniRadParandChakraverty2023}. 
Nevertheless, this literature has focused primarily on kernel construction and predictive accuracy, leaving the mathematical properties of the underlying orthogonal basis unexploited for structural interpretability.
Two recent lines of work in related settings point toward the potential of such decompositions. First, in the context of uncertainty quantification, \cite{ManfrediTrinchero2024} establish an equivalence between polynomial chaos expansions and least-squares SVMs using Hermite and Legendre kernels, showing that the orthogonal expansion coefficients can be recovered directly from the dual solution after model fitting. Second, in nonparametric regression on metric spaces, \cite{WangEtAl2025} construct an analysis-of-variance (ANOVA) decomposition within tensor-product RKHSs to separate main effects and higher-order interactions. In standard soft-margin SVM classification, however, orthogonal polynomial bases have not yet been exploited to decompose the regularized decision function and quantify how model complexity is distributed across polynomial degrees, interaction orders, and explanatory variables.

In this paper, we consider SVMs constructed from tensor products of truncated Christoffel--Darboux kernels \citep{Simon2008,LasserrePauwelsPutinar2022}. Because the associated RKHS admits an explicit finite orthonormal basis of tensor-product polynomials, the regularized component of the fitted decision function has an exact finite expansion, and its squared RKHS norm equals the sum of the squared expansion coefficients. Grouping these coefficients by total polynomial degree, interaction order, and active explanatory variables yields an exact additive decomposition of the squared RKHS norm while leaving the fitted classification rule unchanged.

Building on this decomposition, we introduce a novel post-training analysis methodology, \emph{Orthogonal Representation Contribution Analysis} (ORCA), based on a new set of normalized \emph{Orthogonal Kernel Contribution} (OKC) indices and their graphical representations. These indices quantify the fractions of the squared RKHS norm associated with total polynomial degrees, interaction orders, pure marginal terms, and pairwise interactions, as well as the overall participation of each explanatory variable. ORCA integrates these measures into joint structural maps, variable-specific panels, and comparative summaries capable of distinguishing classifiers that achieve similar predictive accuracy through different internal structures.

The remainder of the paper is organized as follows. Section~\ref{sec:svm-rkhs} recalls the soft-margin SVM formulation, reviews truncated Christoffel--Darboux kernels, and derives the exact finite multi-index orthogonal expansion of the fitted decision function. Section~\ref{sec:okc-indices} introduces the joint, aggregated, and variable-level OKC indices along with their conservation properties. Section~\ref{sec:ORCA} details the six-step ORCA procedure and its graphical and comparative tools. Section~\ref{sec:experiments} illustrates the methodology through two controlled benchmarks: the performance-matched \emph{ORCA Quartet} and the double-spiral problem across five Jacobi kernels. Section~\ref{sec:realdata} presents an application to the Vertebral Column data set under cross-validated hyperparameter selection, and Section~\ref{sec:conclusions} summarizes our main conclusions and outlines directions for future research.


\section{SVMs with truncated orthogonal-polynomial kernels}
\label{sec:svm-rkhs}

This section first reviews the standard soft-margin SVM formulation and the tensor-product construction of truncated orthogonal-polynomial kernels. We then introduce an explicit finite multi-index representation of the tensor-product kernel and use it to derive an exact orthogonal expansion of the RKHS component of the decision function obtained by fitting a soft-margin SVM with this kernel.

\subsection{Background}

We first recall the standard RKHS representation of a soft-margin SVM and the construction of truncated Christoffel--Darboux kernels from tensor products of orthonormal polynomial systems. We begin with the SVM formulation.
Let $\mathcal{X}\subseteq\mathbb{R}^d$ denote the input space, and consider the
binary training sample
\begin{equation}\label{eq:data}
	\{(\mathbf{x}_i,y_i)\}_{i=1}^m,
	\qquad
	\mathbf{x}_i\in\mathcal{X},
	\qquad
	y_i\in\{-1,+1\}.
\end{equation}
Let $K:\mathcal{X}\times\mathcal{X}\to\mathbb{R}$ be a positive semidefinite
kernel with associated RKHS $\mathcal{H}$ and feature map
$\Phi:\mathcal{X}\to\mathcal{H}$, so that
\[
K(\mathbf{x},\mathbf{z})
=
\langle\Phi(\mathbf{x}),\Phi(\mathbf{z})\rangle_{\mathcal{H}}.
\]
We use the standard binary soft-margin SVM formulation
\cite{CortesVapnik1995,Vapnik1998,ScholkopfSmola2002}. Its primal problem is
\begin{align}
	\min_{\mathbf{w},\,b,\,\boldsymbol{\xi}}
	\quad &
	\frac12\|\mathbf{w}\|_{\mathcal{H}}^2
	+
	C\sum_{i=1}^m\xi_i
	\label{eq:svm-primal}\\
	\text{subject to}\quad &
	y_i\big(
	\langle\mathbf{w},\Phi(\mathbf{x}_i)\rangle_{\mathcal{H}}+b
	\big)
	\geq 1-\xi_i,
	\qquad
	\xi_i\geq0,
	\qquad i=1,\dots,m .
	\nonumber
\end{align}

By the representer theorem
\cite{KimeldorfWahba1971,ScholkopfHerbrichSmola2001},
the RKHS component of a solution belongs to the span of the kernel sections.
Hence the resulting decision function can be written as
\begin{equation}\label{eq:decision}
	g(\mathbf{x})
	=
	h(\mathbf{x})+b,
	\qquad
	h(\mathbf{x})
	=
	\sum_{i=1}^m
	\alpha_i y_i K(\mathbf{x}_i,\mathbf{x}),
\end{equation}
and the classifier is
$\widehat{y}(\mathbf{x})=\operatorname{sign}(g(\mathbf{x}))$.
The dual coefficients satisfy
\begin{equation}\label{eq:svm-dual}
	0\leq\alpha_i\leq C,
	\qquad i=1,\dots,m,
	\qquad
	\sum_{i=1}^m\alpha_i y_i=0.
\end{equation}
The latter equality will be relevant when identifying structural zero
contributions in the orthogonal representation.

We now introduce the truncated
Christoffel--Darboux kernel for scalar inputs, which will serve as the univariate
building block for the tensor-product kernel on vector inputs defined below.

Let $I\subset\mathbb{R}$ be an interval and let $\mu$ be a finite positive
Borel measure supported on $I$, with infinite support and finite polynomial
moments of every order. Let $\{p_k\}_{k\geq0}$ denote the corresponding
orthonormal polynomial sequence, normalized with positive leading coefficient:
\begin{equation}\label{eq:OP-orthonormality}
	\int_I p_k(t)p_\ell(t)\,d\mu(t)
	=
	\delta_{k\ell},
	\qquad k,\ell\geq0.
\end{equation}
For background on orthogonal polynomials and Christoffel--Darboux kernels,
see, for example,
\cite{Gautschi2004,Szego1939,Simon2008,LasserrePauwelsPutinar2022}.

For a truncation level $n\in\mathbb{N}_0$, consider the finite
Christoffel--Darboux kernel
\begin{equation}\label{eq:CD-def}
	K_n(x,z)
	=
	\sum_{k=0}^{n}p_k(x)p_k(z),
	\qquad x,z\in I.
\end{equation}
The polynomial space
$\mathcal{H}_n^{(1)}=\operatorname{span}\{p_0,\dots,p_n\}$,
endowed with the inner product inherited from $L^2(\mu)$, is a
finite-dimensional RKHS with reproducing kernel $K_n$.
Thus $\{p_0,\dots,p_n\}$ is an orthonormal basis of
$\mathcal{H}_n^{(1)}$, with
$\dim(\mathcal{H}_n^{(1)})=n+1$.
Equivalently, the corresponding feature map is
\[
\Phi_n(x)
=
\big(p_0(x),p_1(x),\dots,p_n(x)\big)^\top,
\qquad
K_n(x,z)
=
\langle\Phi_n(x),\Phi_n(z)\rangle_{\mathbb{R}^{n+1}}.
\]

To construct the corresponding kernel for vector inputs, we take the tensor
product of the scalar-input kernels across coordinates. Let
$\mathcal{X}=I^d$ and equip $I^d$ with the product measure
$\mu^{\otimes d}$. Following the tensor-product construction used for
orthogonal-polynomial kernels in SVM classification
(see, e.g., \cite{OzerChenCirpan2011,MoghaddamHamidzadeh2016}), we use the same truncation
level $n$ in each coordinate and define
\begin{equation}\label{eq:CD-tensor-def}
	K_n^{(d)}(\mathbf{x},\mathbf{z})
	:=
	\prod_{j=1}^{d}K_n(x_j,z_j),
	\qquad
	\mathbf{x},\mathbf{z}\in I^d.
\end{equation}
In particular, $K_n^{(d)}$ is symmetric and positive semidefinite on
$I^d \times I^d$. 

\subsection{Finite multi-index representation and exact orthogonal expansion}
\label{sec:exact-expansion}

We now introduce an explicit finite multi-index representation of the
tensor-product kernel and use it to derive an exact orthogonal expansion of the
RKHS component of the decision function obtained by fitting a soft-margin SVM
with this kernel. To this end, define the multi-index set
$\mathcal{I}_n^{(d)}:=\{0,1,\dots,n\}^d$, whose cardinality is
$\operatorname{card}\bigl(\mathcal{I}_n^{(d)}\bigr)=(n+1)^d$.
For a multi-index
$\boldsymbol{k}=(k_1,\dots,k_d)\in\mathcal{I}_n^{(d)}$, the component $k_j$
specifies the degree of the univariate orthonormal polynomial associated with
the $j$-th input variable.

Each such multi-index determines a tensor-product polynomial obtained by
combining the corresponding univariate orthonormal polynomials across
coordinates. Specifically, for
$\mathbf{x}=(x_1,\dots,x_d)\in I^d$, define
\begin{equation}\label{eq:tensor-basis}
	p_{\boldsymbol{k}}(\mathbf{x})
	:=
	\prod_{j=1}^d p_{k_j}(x_j).
\end{equation}
By the orthonormality relation in \eqref{eq:OP-orthonormality}, the family
$\{p_{\boldsymbol{k}}\}_{\boldsymbol{k}\in\mathcal{I}_n^{(d)}}$ is orthonormal
in $L^2(\mu^{\otimes d})$. It therefore forms an orthonormal basis of
\begin{equation}\label{eq:Hn-tensor-def}
	\mathcal{H}_n^{(d)}
	:=
	\operatorname{span}
	\left\{
	p_{\boldsymbol{k}}:
	\boldsymbol{k}\in\mathcal{I}_n^{(d)}
	\right\},
	\qquad
	\dim\bigl(\mathcal{H}_n^{(d)}\bigr)
	=
	(n+1)^d,
\end{equation}
where $\mathcal{H}_n^{(d)}$ is endowed with the inner product inherited from
$L^2(\mu^{\otimes d})$.

The preceding basis construction provides the coordinate system needed to
rewrite the tensor-product kernel as a finite sum over multi-indices, as
formalized in the following proposition.

\begin{proposition}
	\label{prop:CD-multiindex}
	The tensor-product truncated orthogonal-polynomial kernel defined in
	\eqref{eq:CD-tensor-def} admits the expansion
	\begin{equation}\label{eq:CD-multiindex}
		K_n^{(d)}(\mathbf{x},\mathbf{z})
		=
		\sum_{\boldsymbol{k}\in\mathcal{I}_n^{(d)}}
		p_{\boldsymbol{k}}(\mathbf{x})
		p_{\boldsymbol{k}}(\mathbf{z}),
		\qquad
		\mathbf{x},\mathbf{z}\in I^d,
	\end{equation}
	where $p_{\boldsymbol{k}}$ is defined in \eqref{eq:tensor-basis}.
\end{proposition}

\begin{proof}
	Substituting \eqref{eq:CD-def} into \eqref{eq:CD-tensor-def} and expanding the
	finite product of sums gives
	\begin{align*}
		K_n^{(d)}(\mathbf{x},\mathbf{z})
		&=
		\prod_{j=1}^d
		\left(
		\sum_{k_j=0}^n
		p_{k_j}(x_j)p_{k_j}(z_j)
		\right) \\
		&=
		\sum_{\boldsymbol{k}\in\mathcal{I}_n^{(d)}}
		\prod_{j=1}^d
		p_{k_j}(x_j)p_{k_j}(z_j) \\
		&=
		\sum_{\boldsymbol{k}\in\mathcal{I}_n^{(d)}}
		p_{\boldsymbol{k}}(\mathbf{x})
		p_{\boldsymbol{k}}(\mathbf{z}),
	\end{align*}
	where the last equality follows from \eqref{eq:tensor-basis}. 
\end{proof}

Proposition~\ref{prop:CD-multiindex} makes the structure of the multivariate
tensor-product kernel explicit. It shows that $K_n^{(d)}$ can be written as a
finite sum of $(n+1)^d$ products of the form
$p_{\boldsymbol{k}}(\mathbf{x})p_{\boldsymbol{k}}(\mathbf{z})$, where
$p_{\boldsymbol{k}}$ ranges over the orthonormal tensor-product polynomial
basis. Each basis function is uniquely indexed by a multi-index
$\boldsymbol{k}\in\mathcal{I}_n^{(d)}$ whose components specify the univariate polynomial
degree associated with each input variable. Since these basis functions form
an orthonormal basis of $\mathcal{H}_n^{(d)}$, the representation also shows
that $K_n^{(d)}$ is the reproducing kernel of this space.

We now use this explicit coordinate representation to characterize the RKHS
component of the decision function of a soft-margin SVM fitted with
$K_n^{(d)}$. From \eqref{eq:decision}, the decision function is
\begin{equation}\label{eq:sec4-2-score}
	g(\mathbf{x})
	=
	h_n^{(d)}(\mathbf{x})+b,
	\qquad\textup{with }
	h_n^{(d)}(\mathbf{x})
	:=
	\sum_{i=1}^m
	\alpha_i y_i
	K_n^{(d)}(\mathbf{x}_i,\mathbf{x}),
\end{equation}
where $(\alpha_1,\dots,\alpha_m)$ is a solution to the SVM dual problem, and $b\in\mathbb{R}$ is the bias. The function
$h_n^{(d)}$ is the RKHS component of the fitted decision function. Under this representation, the regularization term in the primal SVM
formulation \eqref{eq:svm-primal} is one half of the squared RKHS norm of
$h_n^{(d)}$.

Substituting the finite kernel expansion \eqref{eq:CD-multiindex} into
\eqref{eq:sec4-2-score} expresses $h_n^{(d)}$ exactly in the same orthonormal
basis. The resulting expansion and its coefficients are given in the following result.

\begin{proposition}
	\label{prop:sec4-2-finite-expansion}
	Let $h_n^{(d)}$ be the RKHS component defined in
	\eqref{eq:sec4-2-score}. Then $h_n^{(d)}\in\mathcal{H}_n^{(d)}$ and admits the
	unique expansion
	\begin{equation}\label{eq:sec4-2-expansion}
		h_n^{(d)}(\mathbf{x})
		=
		\sum_{\boldsymbol{k}\in\mathcal{I}_n^{(d)}}
		c_{\boldsymbol{k}}\,p_{\boldsymbol{k}}(\mathbf{x}),
	\end{equation}
	where
	\begin{equation}\label{eq:sec4-2-coeffs}
		c_{\boldsymbol{k}}
		=
		\sum_{i=1}^m
		\alpha_i y_i\,p_{\boldsymbol{k}}(\mathbf{x}_i),
		\qquad\textup{for }
		\boldsymbol{k}\in\mathcal{I}_n^{(d)}.
	\end{equation}
\end{proposition}

\begin{proof}
	Substituting \eqref{eq:CD-multiindex} into
	\eqref{eq:sec4-2-score} gives
	\begin{align*}
		h_n^{(d)}(\mathbf{x})
		&=
		\sum_{i=1}^m
		\alpha_i y_i
		\sum_{\boldsymbol{k}\in\mathcal{I}_n^{(d)}}
		p_{\boldsymbol{k}}(\mathbf{x}_i)
		p_{\boldsymbol{k}}(\mathbf{x})\\
		&=
		\sum_{\boldsymbol{k}\in\mathcal{I}_n^{(d)}}
		\left(
		\sum_{i=1}^m
		\alpha_i y_i
		p_{\boldsymbol{k}}(\mathbf{x}_i)
		\right)
		p_{\boldsymbol{k}}(\mathbf{x}),
	\end{align*}
	where the order of summation can be interchanged because all sums are
	finite. 
	This identity shows that $h_n^{(d)}\in\mathcal{H}_n^{(d)}$ and yields
	\eqref{eq:sec4-2-expansion}--\eqref{eq:sec4-2-coeffs}. The expansion is unique
	because
	$\{p_{\boldsymbol{k}}\}_{\boldsymbol{k}\in\mathcal{I}_n^{(d)}}$ is an
	orthonormal basis of $\mathcal{H}_n^{(d)}$.
\end{proof}

Proposition~\ref{prop:sec4-2-finite-expansion} provides the foundation for the
structural analysis developed below.
In particular, the coefficients
$\{c_{\boldsymbol{k}}\}_{\boldsymbol{k}\in\mathcal{I}_n^{(d)}}$ are 
the
coordinates of $h_n^{(d)}$ in the orthonormal basis associated with
$K_n^{(d)}$.

\section{Orthogonal Kernel Contribution indices}
\label{sec:okc-indices}

The exact orthogonal expansion derived in Section~\ref{sec:exact-expansion} provides a direct decomposition of the regularized component of the trained SVM in the coordinates induced by the truncated orthogonal-polynomial kernel. Building on this representation, we introduce a family of indices designed to make the internal structure of the fitted SVM quantitatively accessible and to extract relevant information on how the classifier distributes its complexity across the different components of the orthogonal expansion.

The construction of these indices relies on the additive decomposition of the squared RKHS norm induced by the orthonormal basis.
Since
$\{p_{\boldsymbol{k}}\}_{\boldsymbol{k}\in\mathcal{I}_n^{(d)}}$
is an orthonormal basis of $\mathcal{H}_n^{(d)}$, from \eqref{eq:sec4-2-expansion}, Parseval's identity yields
the exact decomposition
\begin{equation}\label{eq:rkhs-norm-sum-tensor}
	\|h_n^{(d)}\|_{\mathcal{H}_n^{(d)}}^2
	=
	\sum_{\boldsymbol{k}\in\mathcal{I}_n^{(d)}}
	c_{\boldsymbol{k}}^2 .
\end{equation}
Thus each $c_{\boldsymbol{k}}^2$ is the contribution of the orthonormal
tensor-product polynomial $p_{\boldsymbol{k}}$ to the squared RKHS norm of the
fitted component $h_n^{(d)}$.
In this sense, \eqref{eq:rkhs-norm-sum-tensor} provides an exact modewise decomposition of the SVM regularization term in the orthogonal coordinates induced by the truncated tensor-product kernel.
Throughout this article we assume
$\|h_n^{(d)}\|_{\mathcal{H}_n^{(d)}}>0$.

While this modewise decomposition is exact, it is often too fine-grained for direct interpretation in moderate or high dimension, since the number of coefficients is
$\mathrm{card}\big(\mathcal I_n^{(d)}\big)=(n+1)^d$.
For instance, if $n=5$ and $d=6$, then $\mathrm{card}(\mathcal I_n^{(d)})=6^6=46{,}656$, so the squared RKHS norm is distributed over tens of thousands of individual multi-index contributions.
To obtain a more informative representation, we organize these elementary contributions along two complementary dimensions. The first identifies which input variables are involved in each multi-index, thereby describing how the fitted classifier is structured across the input coordinates. The second quantifies the polynomial complexity of the multi-index through its total degree.

Accordingly, for each
$\boldsymbol{k}\in\mathcal I_n^{(d)}$, we define the active-coordinate set
\[
\operatorname{act}(\boldsymbol{k}) := \{\, j\in\{1,\dots,d\}: k_j>0 \,\},
\]
which identifies 
the input variables that are active in the basis function $p_{\boldsymbol{k}}$, while $q(\boldsymbol{k}):= \mathrm{card}\big(\operatorname{act}(\boldsymbol{k})\big)$ gives their number.
Hence, $q(\boldsymbol{k})$ measures the interaction order associated with the term indexed by $\boldsymbol{k}$, so $0\leq q(\boldsymbol{k})\leq d$. In particular, $q(\boldsymbol{k})=0$ corresponds to the constant term, $q(\boldsymbol{k})=1$ to a contribution involving a single input variable, $q(\boldsymbol{k})=2$ to a contribution involving exactly two input variables, and, more generally, $q(\boldsymbol{k})=q$ to a contribution involving exactly $q$ input variables.

To complement this interaction-based classification, we also associate with each multi-index its total polynomial degree defined as follows
\[
	N(\boldsymbol{k})
:=
\sum_{j=1}^{d} k_j,
\qquad
0\leq N(\boldsymbol{k})\leq dn.
\]
This quantity measures the overall polynomial complexity of the tensor-product polynomial $p_{\boldsymbol{k}}$. 
Therefore, every elementary contribution $c_{\boldsymbol{k}}^2$ carries two intrinsic labels: its interaction order $q(\boldsymbol{k})$ and its total degree $N(\boldsymbol{k})$. Thus, the pairs $(N,q)$ provide a natural way to organize the RKHS regularization term.

The admissible pairs $(N,q)$ are determined by the constraints on the multi-index
$\boldsymbol{k}$. For a fixed interaction order $q\geq 1$, exactly $q$
components of $\boldsymbol{k}$ are positive, and each of these components
belongs to $\{1,\dots,n\}$. Therefore, the total degree $N$ must satisfy the
admissibility condition
\begin{equation}\label{eq:admissibility-condition}
	q \leq N \leq qn,
	\qquad q=1,\dots,d.
\end{equation}

The pair $(N,q)=(0,0)$ is also admissible and corresponds to the constant basis function. In the standard soft-margin SVM with an explicit bias, however, this contribution is equal to zero. Since
$p_{\boldsymbol{0}}(\mathbf{x})\equiv\kappa_0$ is constant,
\eqref{eq:sec4-2-coeffs} and the dual feasibility condition
\eqref{eq:svm-dual} imply
\begin{equation}\label{eq:orca-constant-zero}
	c_{\boldsymbol{0}}
	=
	\kappa_0
	\sum_{i=1}^{m}\alpha_i y_i
	=
	0.
\end{equation}

The interaction order $q(\boldsymbol{k})$ and the total polynomial degree
$N(\boldsymbol{k})$ assign each coefficient $c_{\boldsymbol{k}}$ to a unique
admissible pair $(N,q)$. This leads to the following definition.

\begin{definition}
	For each admissible pair $(N,q)$, define the block
	\begin{equation}\label{eq:block-qN}
		\mathcal{B}_{N}^{(q)}
		:=
		\left\{
		\boldsymbol{k}\in\mathcal{I}_n^{(d)}:
		q(\boldsymbol{k})=q,\;
		N(\boldsymbol{k})=N
		\right\}.
	\end{equation}
	The corresponding \emph{Orthogonal Kernel Contribution} (OKC) index is
	\begin{equation}\label{eq:OKC-qN}
		\operatorname{OKC}_{N}^{(q)}
		:=
		\frac{
			\displaystyle
			\sum_{\boldsymbol{k}\in\mathcal{B}_{N}^{(q)}}
			c_{\boldsymbol{k}}^{2}
		}{
			\displaystyle
			\left\|h_n^{(d)}\right\|_{\mathcal{H}_n^{(d)}}^{2}
		}.
	\end{equation}
\end{definition}
Thus, $\operatorname{OKC}_{N}^{(q)}$ measures the fraction of the squared RKHS norm
associated with terms involving exactly $q$ input variables and having total
polynomial degree $N$.
By definition, all OKC indices are nonnegative.
For a non-admissible pair $(N,q)$, the corresponding block
$\mathcal{B}_{N}^{(q)}$ is empty and therefore
$\operatorname{OKC}_{N}^{(q)}=0$ by construction. Such pairs contain no
information about the fitted SVM. 
 Moreover, since the blocks
$\mathcal{B}_{N}^{(q)}$ form a partition of the multi-index set
$\mathcal{I}_n^{(d)}$, the normalization in \eqref{eq:OKC-qN} and
\eqref{eq:rkhs-norm-sum-tensor} imply
\begin{equation}\label{eq:OKC-qN-sum1}
	\sum_{q=0}^{d}\sum_{N=0}^{dn}
	\operatorname{OKC}_{N}^{(q)}
	=
	1.
\end{equation}
Hence, the family
$\{\operatorname{OKC}_{N}^{(q)}\}$
defines a normalized two-dimensional profile of the squared RKHS norm,
jointly resolved by interaction order and total polynomial degree.

The cardinality of each admissible block $\mathcal{B}_{N}^{(q)}$ defined in  \eqref{eq:block-qN} can be obtained directly from the
multi-index constraints. For a fixed pair $(N,q)$ with $q\geq 1$, one first
selects the $q$ active coordinates among the $d$ input variables and then
distributes the total degree $N$ among these coordinates, with each active
component taking values in $\{1,\dots,n\}$. Hence,
\begin{equation}\label{eq:card-block-qN-coeff}
	\operatorname{card}\big(\mathcal{B}_{N}^{(q)}\big)
	=
	\binom{d}{q}
	\,[t^{N}]
	\left(t+t^{2}+\cdots+t^{n}\right)^{q},
\end{equation}
where $[t^{N}]$ denotes the coefficient of $t^{N}$.
Equivalently, by the inclusion--exclusion principle,
\begin{equation}\label{eq:card-block-qN}
	\operatorname{card}\big(\mathcal{B}_{N}^{(q)}\big)
	=
	\binom{d}{q}
	\sum_{j=0}^{\left\lfloor (N-q)/n \right\rfloor}
	(-1)^j
	\binom{q}{j}
	\binom{N-jn-1}{q-1},
	\qquad
	q\leq N\leq qn.
\end{equation}
Here, $\lfloor x \rfloor$ denotes the greatest integer less than or equal to $x$.

The quantity
$\operatorname{card}\big(\mathcal{B}_{N}^{(q)}\big)$
depends only on the structural parameters $(d,n,q,N)$ and gives the number of
orthogonal coefficients contained in the block $(N,q)$. This combinatorial
information complements the value of $\operatorname{OKC}_{N}^{(q)}$.
Indeed, two blocks may have the same value of
$\operatorname{OKC}_{N}^{(q)}$ while containing very different numbers of
coefficients. The same fraction of the squared RKHS norm may therefore be
concentrated in a small number of coefficients in one block and distributed
among many coefficients in another. Thus,
$\operatorname{OKC}_{N}^{(q)}$ measures the total contribution of the block,
whereas $\operatorname{card}\big(\mathcal{B}_{N}^{(q)}\big)$ indicates how
many coefficients contribute to that total.

\subsection{Aggregated OKC indices}
\label{subsec:aggregated-okc}

The indices $\operatorname{OKC}_{N}^{(q)}$ provide a joint decomposition of
the squared RKHS norm according to interaction order and total polynomial
degree. Aggregating these indices along either coordinate yields two
one-dimensional summaries of the fitted classifier: one according to the
number of active input variables and one according to polynomial complexity.

For a fixed interaction order $q\in\{1,\dots,d\}$, the admissible total
degrees satisfy $q\leq N\leq qn$. We define the interaction-order block
\begin{equation}\label{eq:block-q}
	\mathcal{B}^{(q)}
	:=
	\bigcup_{N=q}^{qn}
	\mathcal{B}_{N}^{(q)}
	=
	\left\{
	\boldsymbol{k}\in\mathcal{I}_n^{(d)}:
	q(\boldsymbol{k})=q
	\right\}.
\end{equation}
Thus, $\mathcal{B}^{(q)}$ contains all multi-indices associated with terms
involving exactly $q$ input variables, irrespective of their total polynomial
degree.

Since the blocks $\mathcal{B}_{N}^{(q)}$ are mutually disjoint, the cardinality of
$\mathcal{B}^{(q)}$ is the sum of the block cardinalities in
\eqref{eq:card-block-qN}:
\begin{equation}\label{eq:card-block-q}
	\operatorname{card}\big(\mathcal{B}^{(q)}\big)
	=
	\sum_{N=q}^{qn}
	\operatorname{card}\big(\mathcal{B}_{N}^{(q)}\big)
	=
	\binom{d}{q}n^q,
	\qquad q=1,\dots,d,	
\end{equation}
which gives the number of
orthogonal coefficients associated with interaction order $q$.
For $q=0$, $\mathcal{B}^{(0)}=\mathcal{B}_{0}^{(0)}$ and
$\operatorname{card}\big(\mathcal{B}^{(0)}\big)=1$.

This block structure provides the natural basis for aggregating the elementary
indices $\operatorname{OKC}_{N}^{(q)}$ over all admissible total degrees
associated with a fixed interaction order $q$.

\begin{definition}
	For each interaction order $q\in\{1,\dots,d\}$, the aggregated OKC index
	of interaction order $q$ is defined by
	\begin{equation}\label{eq:I-OKC-q}
		\operatorname{I-OKC}^{(q)}
		:=
		\sum_{N=q}^{qn}
		\operatorname{OKC}_{N}^{(q)}.
	\end{equation}
\end{definition}
For $q=0$, the only admissible pair is $(0,0)$, so 
$	\operatorname{I-OKC}^{(0)}
=
\operatorname{OKC}_{0}^{(0)}$.
In the standard soft-margin SVM with an explicit bias,
\eqref{eq:orca-constant-zero} gives
$\operatorname{I-OKC}^{(0)}=0$.

Thus, $\operatorname{I-OKC}^{(q)}$ measures the fraction of the squared
RKHS norm associated with terms involving exactly $q$ input variables,
regardless of their total polynomial degree.
Specifically, $\operatorname{I-OKC}^{(0)}$ corresponds to the constant component, $\operatorname{I-OKC}^{(1)}$ to purely marginal effects, $\operatorname{I-OKC}^{(2)}$ to pairwise interactions, $\operatorname{I-OKC}^{(3)}$ to three-way interactions, and so on up to $\operatorname{I-OKC}^{(d)}$, which collects the contribution of terms involving all $d$ input variables simultaneously. Accordingly, the interaction-order profile $\big(\operatorname{I-OKC}^{(0)},\operatorname{I-OKC}^{(1)},\dots,\operatorname{I-OKC}^{(d)}\big)$ provides a concise post-training summary of how the RKHS regularization term is distributed across levels of interaction complexity.
Moreover, the normalized interaction-order profile satisfies
\begin{equation}\label{eq:sum-IOKC}
\operatorname{I-OKC}^{(0)}+\operatorname{I-OKC}^{(1)}+\sum_{q=2}^{d}\operatorname{I-OKC}^{(q)}=1,
\end{equation}
which separates the total regularization contribution into constant, purely marginal, and higher-order interaction components.

We next aggregate the OKC indices according to total polynomial degree.
For a fixed $N\in\{1,\dots,dn\}$, the admissibility condition
\eqref{eq:admissibility-condition} gives
\begin{equation}\label{eq:admissible-q-given-N}
	\left\lceil\frac{N}{n}\right\rceil
	\leq q\leq
	\min\{d,N\},
\end{equation}
where $\lceil x\rceil$ denotes the smallest integer greater than or equal to
$x$. We define the total-degree block
\begin{equation}\label{eq:block-N}
	\mathcal{B}_{N}
	:=
	\bigcup_{q=\lceil N/n\rceil}^{\min\{d,N\}}
	\mathcal{B}_{N}^{(q)}
	=
	\left\{
	\boldsymbol{k}\in\mathcal{I}_n^{(d)}:
	N(\boldsymbol{k})=N
	\right\}.
\end{equation}
Thus, $\mathcal{B}_{N}$ contains all multi-indices associated with terms of
total polynomial degree $N$, irrespective of the number of active input
variables.

As before, since the union in \eqref{eq:block-N} is disjoint, the cardinality of
$\mathcal{B}_{N}$ is obtained by summing the cardinalities of the corresponding
blocks $\mathcal{B}_{N}^{(q)}$. Hence, for $N=1,\dots,dn$,
\begin{align}
	\operatorname{card}\big(\mathcal{B}_{N}\big)
	&=
	\sum_{q=\lceil N/n\rceil}^{\min\{d,N\}}
	\operatorname{card}\big(\mathcal{B}_{N}^{(q)}\big)
	\nonumber
	=
	[t^N]\left(1+t+\cdots+t^n\right)^d,
	\label{eq:card-block-N}
\end{align}
where the second equality follows from \eqref{eq:card-block-qN-coeff}.
For $N=0$, $\mathcal{B}_{0}=\mathcal{B}_{0}^{(0)}$ and
$\operatorname{card}\big(\mathcal{B}_{0}\big)=1$.
Hence, $\operatorname{card}\big(\mathcal{B}_{N}\big)$ gives the number of
orthogonal coefficients having total polynomial degree $N$. In contrast with
$\operatorname{card}\big(\mathcal{B}^{(q)}\big)$, which counts coefficients
according to the number of active input variables,
$\operatorname{card}\big(\mathcal{B}_{N}\big)$ counts them according to their
total polynomial complexity.

\begin{definition}
	For each total polynomial degree $N\in\{1,\dots,dn\}$, the aggregated OKC
	index of total degree $N$ is defined by
	\begin{equation}\label{eq:D-OKC-N}
		\operatorname{D-OKC}_{N}
		:=
		\sum_{q=\lceil N/n\rceil}^{\min\{d,N\}}
		\operatorname{OKC}_{N}^{(q)}.
	\end{equation}
\end{definition}
For $N=0$, the only admissible pair is $(0,0)$, and hence
$\operatorname{D-OKC}_{0}
=
\operatorname{OKC}_{0}^{(0)}$.
In the standard soft-margin SVM with an explicit bias,
$\operatorname{D-OKC}_{0}=0$.

Since the indices $\operatorname{OKC}_{N}^{(q)}$ are nonnegative and form a
normalized decomposition of the squared RKHS norm, the aggregated degree
indices satisfy
\begin{equation}\label{eq:D-OKC-properties}
	0\leq \operatorname{D-OKC}_{N}\leq 1,
	\qquad N=0,\dots,dn,
	\qquad\text{and}\qquad
	\sum_{N=0}^{dn}\operatorname{D-OKC}_{N}=1.
\end{equation}

Hence, $\operatorname{D-OKC}_{N}$ quantifies the fraction of the squared RKHS
norm associated with terms of total polynomial degree $N$, regardless of
their interaction order. Specifically, $\operatorname{D-OKC}_{0}$
corresponds to the constant contribution, $\operatorname{D-OKC}_{1}$ collects
all linear contributions, $\operatorname{D-OKC}_{2}$ all quadratic
contributions, $\operatorname{D-OKC}_{3}$ all cubic contributions, and so on
up to $\operatorname{D-OKC}_{dn}$, which collects the contributions of
highest total degree allowed by the truncation level $n$ in dimension $d$.


\subsection{Input-variable OKC indices}
\label{subsec:variable-okc}

The aggregated indices introduced in Section~\ref{subsec:aggregated-okc}
summarize the squared RKHS norm according to interaction order or total
polynomial degree, without retaining the identity of the input variables
involved. We now refine the OKC decomposition by keeping track of the specific
input variables associated with each term of the orthogonal expansion.


For a single input variable $x_j$, $j\in\{1,\dots,d\}$, the corresponding
marginal block is
\begin{equation}\label{eq:block-marginal-j}
	\mathcal{B}_{\{j\}}
	:=
	\left\{
	\boldsymbol{k}\in\mathcal{I}_n^{(d)}:
	\operatorname{act}(\boldsymbol{k})=\{j\}
	\right\}.
\end{equation}
This block contains all terms depending on $x_j$ alone. Its cardinality is $\operatorname{card}\big(\mathcal{B}_{\{j\}}\big)
=n$,
corresponding to the polynomial degrees $1,\dots,n$ in the $j$-th coordinate.

\begin{definition}
	For each input variable $x_j$, $j=1,\dots,d$, the
	\emph{marginal OKC index} is defined by
	\begin{equation}\label{eq:M-OKC-j}
		\operatorname{M-OKC}_{j}
		:=
		\frac{
			\displaystyle
			\sum_{\boldsymbol{k}\in\mathcal{B}_{\{j\}}}
			c_{\boldsymbol{k}}^{2}
		}{
			\displaystyle
			\left\|h_n^{(d)}\right\|_{\mathcal{H}_n^{(d)}}^{2}
		}.
	\end{equation}
\end{definition}

The index $\operatorname{M-OKC}_{j}$ measures the fraction of the squared
RKHS norm associated with terms that depend exclusively on the input variable
$x_j$. Hence, it quantifies the pure marginal contribution of $x_j$, excluding
all terms in which $x_j$ interacts with other input variables. In particular,
\begin{equation}\label{eq:M-OKC-properties}
	0\leq \operatorname{M-OKC}_{j}\leq 1,
	\qquad j=1,\dots,d,
\end{equation}
and the marginal indices recover the first-order interaction contribution:
\begin{equation}\label{eq:M-OKC-sum}
	\sum_{j=1}^{d}
	\operatorname{M-OKC}_{j}
	=
	\operatorname{I-OKC}^{(1)}.
\end{equation}


For two distinct input variables $x_j$ and $x_\ell$, with
$1\leq j<\ell\leq d$, define the pairwise interaction block
\begin{equation}\label{eq:block-pairwise-jl}
	\mathcal{B}_{\{j,\ell\}}
	:=
	\left\{
	\boldsymbol{k}\in\mathcal{I}_n^{(d)}:
	\operatorname{act}(\boldsymbol{k})=\{j,\ell\}
	\right\}.
\end{equation}
This block contains all terms involving exactly the two input variables
$x_j$ and $x_\ell$, with no contribution from any other input variable.
Its cardinality is
$\operatorname{card}\big(\mathcal{B}_{\{j,\ell\}}\big)
=n^2$,
since the two active polynomial degrees can independently take any value in
$\{1,\dots,n\}$.

\begin{definition}
	For each pair of input variables $(x_j,x_\ell)$, with
	$1\leq j<\ell\leq d$, the \emph{pairwise OKC index} is defined by
	\begin{equation}\label{eq:P-OKC-jl}
		\operatorname{P-OKC}_{j\ell}
		:=
		\frac{
			\displaystyle
			\sum_{\boldsymbol{k}\in\mathcal{B}_{\{j,\ell\}}}
			c_{\boldsymbol{k}}^{2}
		}{
			\displaystyle
			\left\|h_n^{(d)}\right\|_{\mathcal{H}_n^{(d)}}^{2}
		}.
	\end{equation}
\end{definition}

The index $\operatorname{P-OKC}_{j\ell}$ measures the fraction of the squared
RKHS norm associated with terms involving exactly the pair
$(x_j,x_\ell)$. It therefore isolates the pure pairwise interaction
contribution of these two input variables. Moreover,
\begin{equation}\label{eq:P-OKC-properties}
	0\leq \operatorname{P-OKC}_{j\ell}\leq 1,
	\qquad 1\leq j<\ell\leq d,
\end{equation}
and summing over all input-variable pairs gives
\begin{equation}\label{eq:P-OKC-sum}
	\sum_{1\leq j<\ell\leq d}
	\operatorname{P-OKC}_{j\ell}
	=
	\operatorname{I-OKC}^{(2)}.
\end{equation}

The same construction extends directly to higher-order interactions.
For any
subset $S\subseteq\{1,\dots,d\}$ with $\operatorname{card}(S)=q\geq 3$, one
may define a corresponding $q$-way OKC index by summing the normalized
contributions over all multi-indices satisfying
$\operatorname{act}(\boldsymbol{k})=S$.


The marginal index $\operatorname{M-OKC}_{j}$ measures the contribution of terms that depend exclusively on the input variable $x_j$. To quantify the overall contribution of an input
variable, including all interactions in which it participates, we define the
global block
\begin{equation}\label{eq:block-global-j}
	\mathcal{B}_{j}^{\mathrm{glob}}
	:=
	\left\{
	\boldsymbol{k}\in\mathcal{I}_n^{(d)}:
	j\in\operatorname{act}(\boldsymbol{k})
	\right\},
	\qquad j=1,\dots,d.
\end{equation}
Thus, $\mathcal{B}_{j}^{\mathrm{glob}}$ contains the marginal terms associated
with $x_j$ together with all pairwise and higher-order interaction terms in
which $x_j$ is involved.

Its cardinality is
\begin{equation}\label{eq:card-block-global-j}
	\operatorname{card}\big(\mathcal{B}_{j}^{\mathrm{glob}}\big)
	=
	n(n+1)^{d-1},
\end{equation}
since the $j$-th component of $\boldsymbol{k}$ can take any value in
$\{1,\dots,n\}$, whereas each of the remaining components can take any value
in $\{0,\dots,n\}$.

\begin{definition}
	For each input variable $x_j$, $j=1,\dots,d$, the
	\emph{global OKC index} is defined by
	\begin{equation}\label{eq:G-OKC-j}
		\operatorname{G-OKC}_{j}
		:=
		\frac{
			\displaystyle
			\sum_{\boldsymbol{k}\in\mathcal{B}_{j}^{\mathrm{glob}}}
			c_{\boldsymbol{k}}^{2}
		}{
			\displaystyle
			\left\|h_n^{(d)}\right\|_{\mathcal{H}_n^{(d)}}^{2}
		}.
	\end{equation}
\end{definition}

The index $\operatorname{G-OKC}_{j}$ measures the fraction of the squared
RKHS norm associated with all terms in which the input variable $x_j$
participates. It therefore includes its marginal contribution
$\operatorname{M-OKC}_{j}$, its pairwise contributions
$\operatorname{P-OKC}_{j\ell}$, and all higher-order interactions involving
$x_j$. In particular,
\begin{equation}\label{eq:G-OKC-properties}
	0\leq \operatorname{G-OKC}_{j}\leq 1,
	\qquad j=1,\dots,d.
\end{equation}

A term involving $q$ input variables contributes to the global index of
each of those $q$ variables. Consequently, the global indices satisfy
\begin{equation}\label{eq:G-OKC-sum}
	\sum_{j=1}^{d}
	\operatorname{G-OKC}_{j}
	=
	\sum_{q=1}^{d}
	q\,\operatorname{I-OKC}^{(q)}.
\end{equation}

All interaction-order indices are nonnegative, and the nonconstant
interaction-order indices sum to one, as stated in
\eqref{eq:sum-IOKC}. Hence, the right-hand side of
\eqref{eq:G-OKC-sum} is a weighted average of the interaction orders
$1,\dots,d$, and therefore
\begin{equation}\label{eq:G-OKC-bounds}
	1
	\leq
	\sum_{j=1}^{d}\operatorname{G-OKC}_{j}
	\leq
	d.
\end{equation}
The lower bound is attained when the squared RKHS norm is entirely concentrated
in marginal contributions, that is, at interaction order $q=1$, where each
term involves a single input variable. The upper bound is attained when the
squared RKHS norm is entirely concentrated at interaction order $q=d$, where
each contributing term involves all $d$ input variables simultaneously.
Thus, the collection
$\{\operatorname{G-OKC}_{j}\}_{j=1}^{d}$ provides a variable-wise measure of
participation rather than a normalized decomposition.

\section{Orthogonal Representation Contribution Analysis}
\label{sec:ORCA}

The OKC indices introduced in Section~\ref{sec:okc-indices} provide the basis
for a post-training analysis of the internal structure of the fitted SVM. We
refer to this procedure as \emph{Orthogonal Representation Contribution
	Analysis} (ORCA). Its purpose is to describe how the squared RKHS norm of the
regularized component of the classifier is distributed across interaction
orders, polynomial degrees, and input variables.

ORCA is incorporated into the SVM workflow after model specification and
training, providing a post-training characterization of the fitted classifier.
Given a dataset with $d$ input variables, the observations are partitioned
into training and test subsets.
	The training subset is used to fit the SVM and to compute the corresponding
	OKC indices and post-training descriptors. The test subset is kept separate
	and is used to assess predictive performance.

The input to ORCA is a trained SVM with a truncated
orthogonal-polynomial kernel. Its output is a standard structural report
containing a joint degree--interaction map, a coordinated input-variable
panel, and a numerical decomposition table. ORCA leaves the fitted classifier
unchanged and uses the OKC indices only to describe its internal
representation. The test subset is reserved for assessing predictive
performance and does not enter the calculation of the indices.

\subsection{The ORCA procedure}

ORCA consists of the following six steps.

\begin{itemize}
	
	\item \textbf{Step 1: SVM specification and training.}
	
	Fix the set of $d$ input variables, their scaling, the truncated
	orthogonal-polynomial kernel, the regularization parameter $C$, the
	truncation level $n$, the orthogonal-polynomial family, and any parameters
	associated with that family. Partition the observations into training and
	test subsets and fit the SVM on the training subset. This produces the
	optimal dual coefficients and the bias term appearing in
	\eqref{eq:sec4-2-score}.
	
	ORCA must be applied to the final selected model. All choices concerning
	the kernel and the SVM are therefore completed before the structural
	analysis begins. The test subset is used after training to evaluate the
	predictive performance of this model.
	
	\item \textbf{Step 2: Recovery of the orthogonal representation.}
	
	Use the fitted dual coefficients to obtain the orthogonal coefficients
	defined in \eqref{eq:sec4-2-coeffs}. By Proposition~
	\ref{prop:sec4-2-finite-expansion}, these coefficients determine the exact
	expansion of the RKHS component of the fitted decision function in the
	tensor-product basis introduced in Section~\ref{sec:exact-expansion}.
	
	This step converts the fitted SVM from its kernel representation into the
	orthogonal coordinate representation used by ORCA. All subsequent
	indices are obtained by grouping the squared coefficients of this
	expansion.
	
	\item \textbf{Step 3: Internal verification.}
	
	Compute the complete collection of
	$\operatorname{OKC}_{N}^{(q)}$ values using \eqref{eq:OKC-qN}. Each
	coefficient is assigned to the block determined by its interaction order
	$q(\boldsymbol{k})$ and total polynomial degree
	$N(\boldsymbol{k})$. The admissible combinations of $N$ and $q$ are
	specified by \eqref{eq:admissibility-condition}.
	
	The implementation must automatically verify the conservation identities
	satisfied by the OKC decomposition. The joint contributions must satisfy
	\eqref{eq:OKC-qN-sum1}; the interaction-order profile must satisfy
	\eqref{eq:sum-IOKC}; and the total-degree profile must satisfy
	\eqref{eq:D-OKC-properties}. The variable-level calculations must also
	satisfy the identities in \eqref{eq:M-OKC-sum},
	\eqref{eq:P-OKC-sum}, and \eqref{eq:G-OKC-sum}.
	
	These checks verify the recovery of the orthogonal coefficients, their
	assignment to the corresponding blocks, and the normalization of the
	resulting indices. A failed identity indicates an error in one of these
	operations and must be resolved before the ORCA report is interpreted.
	
	\item \textbf{Step 4: Construction of the joint structural map.}
	
	Arrange the values of $\operatorname{OKC}_{N}^{(q)}$ according to total
	degree and interaction order. This produces the \emph{ORCA Structural
		Map}, which is the main graphical output of the procedure.
	
	The horizontal axis of this map represents the total polynomial degree
	$N$, and the vertical axis represents the interaction order $q$. The
	color of each admissible cell represents the corresponding value of
	$\operatorname{OKC}_{N}^{(q)}$. Cells excluded by
	\eqref{eq:admissibility-condition} are shaded in gray because they cannot
	contain any coefficient of the expansion.
	
	The map also includes the two aggregated profiles defined in
	Section~\ref{subsec:aggregated-okc}. Bars for
	$\operatorname{D-OKC}_{N}$ are placed along the lower margin, and bars for
	$\operatorname{I-OKC}^{(q)}$ are placed along the right margin. The
	figure therefore presents the joint decomposition and its two
	aggregations in a single display. It identifies which combinations of
	polynomial degree and interaction order support the regularized component
	of the fitted SVM.
	
	\item \textbf{Step 5: Analysis by degree and interaction order.}
	
	Read the joint structural map before examining its aggregated profiles.
	The values of $\operatorname{D-OKC}_{N}$ show the polynomial degrees at
	which the squared RKHS norm is concentrated. The values of
	$\operatorname{I-OKC}^{(q)}$ show whether this contribution is mainly
	associated with marginal terms, pairwise interactions, or higher-order
	interactions.
	
	This step can identify a predominantly marginal representation
	concentrated at low degrees, a marginal representation involving higher
	degrees, an interaction-based representation concentrated at low
	degrees, an interaction-based representation involving higher degrees,
	or a representation distributed across several degree and interaction
	levels.
	
	
		The graphical analysis is complemented by a few scalar summaries. The mean
	interaction order is defined as
	\[
	\bar q
	=
	\sum_{q=1}^{d}
	q\,\operatorname{I-OKC}^{(q)}
	=
	\sum_{j=1}^{d}\operatorname{G-OKC}_{j},
	\]
	and measures the average number of input variables involved in the orthogonal
	terms, weighted by their contribution to the squared RKHS norm.
	Values of $\bar q$ close to $1$ indicate a predominance of marginal
	contributions. Increasing values indicate a greater overall contribution from
	higher interaction orders. Since $\bar q$ summarizes the interaction-order
	distribution through a single weighted average, it is interpreted together
	with the profile
	$\{\operatorname{I-OKC}^{(q)}\}_{q=1}^{d}$. Thus, $\bar q$ provides a
	compact measure of the effective interaction complexity of the fitted SVM.

	The dominant total degree is defined as 
	\[ 
	N_{\max} 
	= 
	\underset{0\leq N\leq dn}{\operatorname{arg\,max}} 
	\operatorname{D-OKC}_{N}, 
	\] 
	that is, the degree that concentrates the largest individual fraction of the 
	squared RKHS norm. 
	
	To complement this local measure of polynomial complexity, we also consider
	effective degree thresholds based on cumulative contribution. For
	$\alpha\in\{0.90,0.95,0.99\}$, define
	\[
	N_{\alpha}
	=
	\min\left\{
	N:
	\sum_{r=0}^{N}\operatorname{D-OKC}_{r}\geq \alpha
	\right\}.
	\]
	Thus, $N_{0.90}$, $N_{0.95}$, and $N_{0.99}$ indicate the smallest total
	polynomial degrees required to account for at least $90\%$, $95\%$, and
	$99\%$ of the squared RKHS norm, respectively.
	
	These quantities provide complementary information. For example, two fitted
	SVMs may have the same dominant degree, say $N_{\max}=4$, while having
	$N_{0.95}=5$ and $N_{0.95}=12$, respectively. In both cases, degree $4$
	carries the largest individual contribution, but the second classifier
	exhibits a substantially more extended tail of polynomial complexity.

	\item \textbf{Step 6: Identification of the role of the input variables.}
	
	After the global degree--interaction structure has been examined, use the
	indices defined in Section~\ref{subsec:variable-okc} to determine the role
	of each input variable. The corresponding graphical output is the
	\emph{ORCA Variable Mechanism Panel}. It contains two coordinated panels
	with a common ordering of the input variables.
	
	The first panel displays $\operatorname{M-OKC}_{j}$ as a bar and
	$\operatorname{G-OKC}_{j}$ as a superimposed point or marker for every
	input variable $x_j$. The marginal index measures the contribution of
	terms that depend exclusively on $x_j$, whereas the global index includes
	every term in which $x_j$ participates. A small separation between the
	bar and the marker indicates that the variable is used mainly through
	marginal terms. A large separation indicates substantial participation
	through interactions. No additional index is formed from this graphical
	separation.
	
	This panel distinguishes variables with a dominant marginal contribution,
	variables used mainly through interactions, variables with both types of
	contribution, and variables with limited participation in the fitted
	representation.
	
	The second panel displays the pairwise indices
	$\operatorname{P-OKC}_{j\ell}$ as a triangular matrix. Its rows and
	columns correspond to the input variables, and the color of each cell
	represents the contribution associated with that pair. The variables must
	have the same order in both panels. Ordering them by decreasing
	$\operatorname{G-OKC}_{j}$ places the variables with the largest overall
	participation first.
	
	The pairwise matrix identifies which pairs account for the bivariate
	part of the interaction structure. When
	$\operatorname{I-OKC}^{(q)}$ is appreciable for $q\geq3$, higher-order
	terms also contribute to the difference between the marginal and global
	participation of a variable. The pairwise matrix then describes one part
	of its interaction mechanism, while the interaction-order profile records
	the remaining higher-order contribution.

\end{itemize}

\subsection{Comparative ORCA}

ORCA can be used to compare the internal representations of different fitted
SVMs. A valid comparison requires the same input variables and scaling, the
same dimension $d$, the same truncation level $n$, the same training and test
partitions, and the same convention for constructing the orthogonal basis.
The ORCA report is computed separately for each fitted model.

For each model, the OKC indices are obtained from the classifier fitted on the
training subset, while predictive performance is evaluated on the common test
subset. These quantities describe different aspects of the fitted classifier:
test accuracy measures predictive performance, whereas the OKC indices describe
the distribution of the squared RKHS norm across polynomial degrees, interaction
orders, and input variables. Models with similar test accuracy may therefore
have different internal structures.

To compare classifiers based on different orthogonal-polynomial families, we
use an \emph{ORCA comparative radar}. The values of $d$ and $n$ are kept fixed,
so that the admissible OKC blocks and their cardinalities are the same for all
models. When the aim is to isolate the effect of the orthogonal-polynomial
family, the regularization parameter $C$ is also fixed, and only the polynomial
family and its parameters are changed.

The radar combines test accuracy with selected summaries of the interaction and
degree structures. 
For
$d\geq3$,
the interaction component is represented by
$\operatorname{I-OKC}^{(1)}$, $\operatorname{I-OKC}^{(2)}$, and, 
\[
\sum_{q=3}^{d}\operatorname{I-OKC}^{(q)}.
\]
Polynomial degree is summarized by $N_{\max}$ and $N_{0.95}$. For the radar,
these quantities are rescaled as
\[
\frac{N_{\max}}{dn},
\qquad
\frac{N_{0.95}}{dn},
\]
so that all coordinates lie in $[0,1]$. Since $d$ and $n$ are common to all
models, this rescaling preserves their relative comparison. 
For $d=2$, $\operatorname{I-OKC}^{(1)}$ and
$\operatorname{I-OKC}^{(2)}$ are complementary, since their sum is one.
In this case, only $\operatorname{I-OKC}^{(2)}$ is displayed in the radar
to avoid redundant information.

\section{Experiments}
\label{sec:experiments}

This section illustrates the structural information provided by ORCA in two
controlled classification settings. The first experiment introduces the
\emph{ORCA Quartet}, a collection of four two-dimensional classification
problems designed to have the same predictive performance under the same SVM
specification while producing different fitted structures. This experiment is
used to examine whether the OKC indices distinguish classifiers that are
indistinguishable in terms of standard predictive metrics. The second
experiment considers the spiral classification problem and examines how the
structure of the fitted classifier changes when different kernels from the
same orthogonal-polynomial family are used on the same data.


We use Jacobi tensor-product kernels throughout. Recall that, for parameters
$\alpha,\beta>-1$, the Jacobi polynomials
$P_k^{(\alpha,\beta)}$ are orthogonal on $[-1,1]$. In the experiments, we
use the normalized Jacobi measure
\[
d\mu^{(\alpha,\beta)}(x)
=
\frac{\Gamma(\alpha+\beta+2)}
{2^{\alpha+\beta+1}\Gamma(\alpha+1)\Gamma(\beta+1)}
(1-x)^\alpha(1+x)^\beta\,dx,
\qquad -1<x<1,
\]
so that
\[
\int_{-1}^{1}d\mu^{(\alpha,\beta)}(x)=1.
\]
For background on Jacobi polynomials and their orthogonality properties, see,
for example, \cite{Szego1939}.

With this normalization, their squared $L^2(\mu^{(\alpha,\beta)})$-norm is
\[
\bigl\|P_k^{(\alpha,\beta)}\bigr\|^2
=
\frac{
	\Gamma(\alpha+\beta+2)
	\Gamma(k+\alpha+1)
	\Gamma(k+\beta+1)
}{
	(2k+\alpha+\beta+1)
	\Gamma(\alpha+1)
	\Gamma(\beta+1)
	\Gamma(k+1)
	\Gamma(k+\alpha+\beta+1)
}.
\]
The corresponding orthonormal polynomials are
\[
p_k^{(\alpha,\beta)}
=
\frac{P_k^{(\alpha,\beta)}}
{\|P_k^{(\alpha,\beta)}\|}.
\]
For the normalized Jacobi measure introduced above, the corresponding
orthonormal polynomials satisfy a three-term recurrence relation. In particular,
\[
p_0^{(\alpha,\beta)}(x)=1,
\]
and
\[
p_1^{(\alpha,\beta)}(x)
=
\frac{1}{2}
\sqrt{
	\frac{\alpha+\beta+3}
	{(\alpha+1)(\beta+1)}
}
\left[
(\alpha-\beta)+(\alpha+\beta+2)x
\right].
\]
When $\alpha=\beta$, the Jacobi system has the usual parity property
$p_k^{(\alpha,\alpha)}(-x)=(-1)^k p_k^{(\alpha,\alpha)}(x)$, whereas for
$\alpha\neq\beta$ this symmetry is lost.


Several classical families are obtained as particular cases of the Jacobi
family. For $\alpha=\beta=0$, the normalized measure is
\[
d\mu^{(0,0)}(x)=\frac12\,dx,
\]
and the Legendre polynomials satisfy
\[
\bigl\|P_k\bigr\|^2=\frac{1}{2k+1},
\qquad
p_k(x)=\sqrt{2k+1}\,P_k(x).
\]
The choices $\alpha=\beta=-1/2$ and $\alpha=\beta=1/2$ give, up to
normalization, the Chebyshev polynomials of the first and second kind,
respectively. Their normalized measures are
\[
d\mu^{(-1/2,-1/2)}(x)
=
\frac{dx}{\pi\sqrt{1-x^2}}
\qquad\textup{and}\qquad
d\mu^{(1/2,1/2)}(x)
=
\frac{2}{\pi}\sqrt{1-x^2}\,dx.
\]
More generally, the Gegenbauer polynomials correspond, up to normalization,
to the symmetric Jacobi case
\[
\alpha=\beta=\lambda-\frac12.
\]
Thus, the Legendre polynomials are recovered for $\lambda=1/2$, while the
Chebyshev polynomials of the second kind correspond to $\lambda=1$.


In all the experiments, the SVM is first trained with the specified Jacobi
tensor-product kernel and the predictive performance is evaluated on the
corresponding test sample. ORCA is then applied to the fitted classifier,
following the procedure described in Section~\ref{sec:ORCA}. The resulting
OKC indices are therefore interpreted as a post-training description of the
orthogonal representation of the fitted SVM.

\subsection{ORCA Quartet: a controlled structural benchmark}
\label{subsec:orca-quartet}

The first experiment is designed to isolate structural differences between
fitted SVMs from differences in predictive performance. We consider four
two-dimensional binary classification data sets, denoted QDS1--QDS4
(\emph{Quartet Data Set} 1--4), with input variables
$x_1,x_2\in[-1,1]$. The four data sets have different generating structures:
QDS1 is additive and linear, QDS2 is additive and nonlinear, QDS3 combines
marginal and pairwise-interaction terms, and QDS4 contains only a pairwise
interaction.

For the four data sets we use the Legendre case of the Jacobi family,
corresponding to $\alpha=\beta=0$. The training samples contain 1200
observations, with 600 observations from each class. The labels were obtained
by thresholding the following polynomial expressions at their corresponding
empirical medians:
\[
\begin{aligned}
	\mathrm{QDS1}:&\qquad
	p_1(x_1)+0.8p_1(x_2),\\
	\mathrm{QDS2}:&\qquad
	p_2(x_1)+0.8p_1(x_2),\\
	\mathrm{QDS3}:&\qquad
	0.8p_1(x_1)+0.8p_1(x_2)
	+0.8p_1(x_1)p_1(x_2),\\
	\mathrm{QDS4}:&\qquad
	p_1(x_1)p_1(x_2).
\end{aligned}
\]
The use of the empirical median produces balanced training samples while
preserving the intended generating structure. Figure~\ref{fig:quartet-data}
shows the four training samples together with the decision boundaries of the
fitted SVMs.

\begin{figure}[!ht]
	\centering
	\includegraphics[width=0.8\textwidth]
	{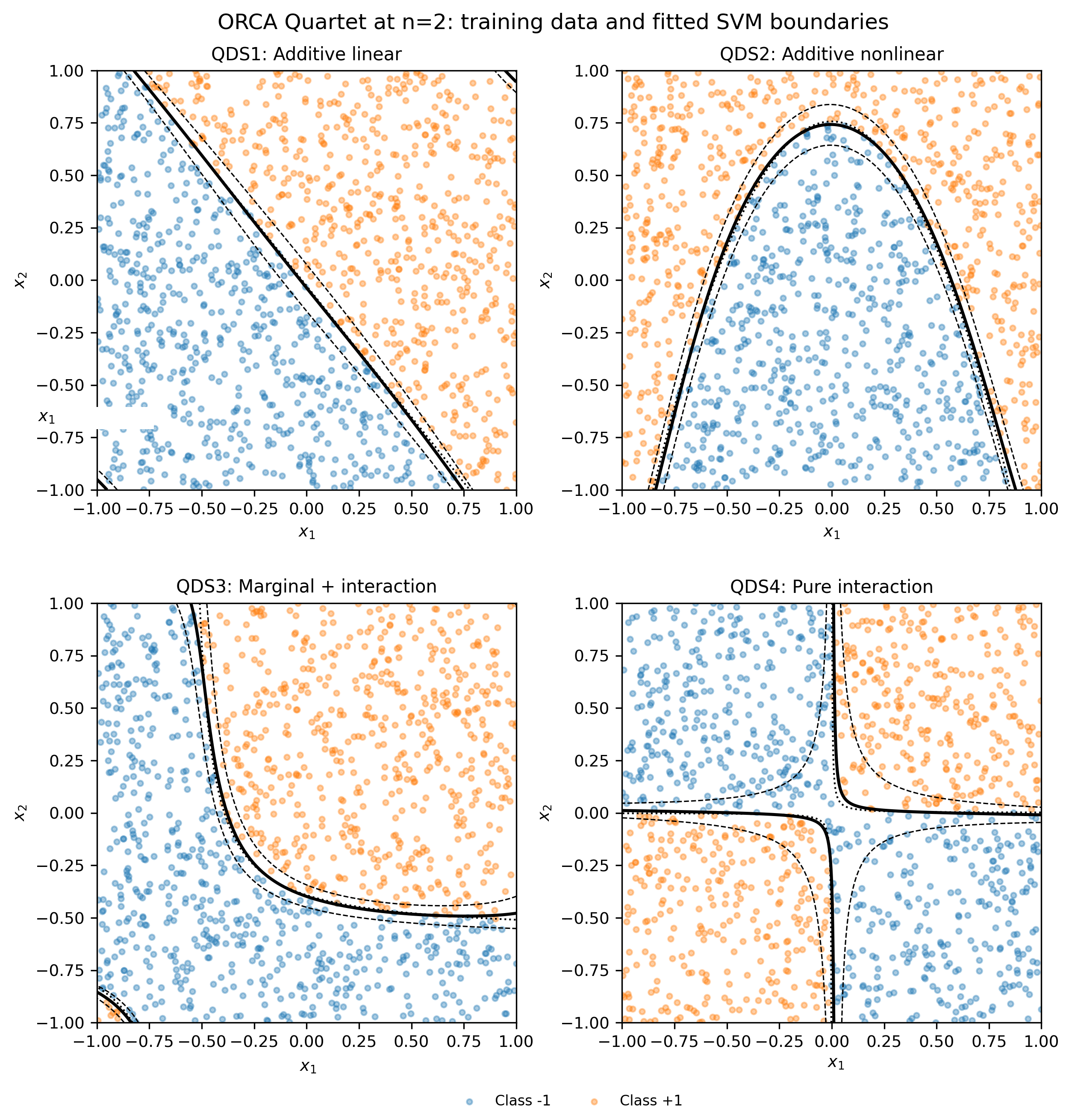}
	\caption{Training samples and fitted SVM decision boundaries for QDS1--QDS4.
		Solid line: fitted SVM decision boundary $g(\mathbf{x})=0$;
		dashed lines: SVM margins $g(\mathbf{x})=\pm1$;
		dotted line(s): boundary used to generate the class labels.}
	\label{fig:quartet-data}
\end{figure}

The SVM specification is kept fixed throughout the experiment. We use the
tensor-product kernel in \eqref{eq:CD-tensor-def} with $d=2$ and truncation
level $n=2$, so that the corresponding feature space has dimension $9$.
The regularization parameter is $C=1$, and the numerical optimization
tolerance is $10^{-9}$. Since the predictors already belong to $[-1,1]$,
no additional scaling is applied. The same fitting procedure is used for
QDS1--QDS4. Computationally, the SVM is fitted on the exact finite
orthogonal feature representation associated with the kernel; this is
equivalent to fitting the kernel SVM defined in
\eqref{eq:svm-primal}--\eqref{eq:decision}.

A controlled test sample of 1000 observations is constructed for each data
set. Its purpose is to remove predictive performance as a possible
explanation for the structural differences observed below. After fitting the
SVM on the corresponding training sample, a common deterministic procedure
is used to select 500 observations on each side of the fitted decision
boundary from a larger candidate sample. The selected observations are
ordered by their decision values and assigned labels according to the same
rank pattern in the four cases. This gives, for every classifier,
\[
\mathrm{TN}=450,\qquad
\mathrm{FP}=50,\qquad
\mathrm{FN}=50,\qquad
\mathrm{TP}=450,
\]
and therefore
\[
\mathrm{Accuracy}
=
\mathrm{Precision}
=
\mathrm{Recall}
=
F_1
=
0.900,
\qquad
\mathrm{AUC}=0.950.
\]
Thus, the test samples are performance-matched by construction. They are
used here as a controlled benchmark for structural comparison.
However, the fitted solutions already differ at the level of support vectors, since QDS1 through QDS4 contain 76, 124, 95, and 185, respectively.

ORCA is then applied to each fitted classifier following
Section~\ref{sec:ORCA}. The complete orthogonal representation is recovered
from the fitted SVM and the indices in \eqref{eq:OKC-qN} are computed before
forming the degree, interaction, and variable profiles defined previously.
All conservation identities required by the ORCA procedure are satisfied to
numerical precision; the largest absolute residual among the internal checks
is below $2.7\times10^{-13}$.

Table~\ref{tab:quartet-global} reports the degree and interaction-order
profiles. Although the predictive metrics are identical, the distributions
of the squared RKHS norm are clearly different.
Figure~\ref{fig:quartet-orca} shows the corresponding ORCA Structural Maps.

\begin{table}[!ht]
	\centering
	\caption{Degree and interaction-order structure of the fitted SVMs for the
		ORCA Quartet Data Set. All models use the same Legendre tensor-product kernel,
		$n=2$, and $C=1$.}
	\label{tab:quartet-global}
	\begin{tabular}{lccccccc}
		\toprule
		&
		$\operatorname{D-OKC}_{1}$ &
		$\operatorname{D-OKC}_{2}$ &
		$\operatorname{D-OKC}_{3}$ &
		$\operatorname{D-OKC}_{4}$ &
		$\operatorname{I-OKC}^{(1)}$ &
		$\operatorname{I-OKC}^{(2)}$ &
		$\bar q$
		\\
		\midrule
		QDS1 & 0.8144 & 0.0010 & 0.1844 & 0.0001
		& 0.8144 & 0.1856 & 1.1856 \\
		QDS2 & 0.3893 & 0.6066 & 0.0037 & 0.0003
		& 0.9955 & 0.0045 & 1.0045 \\
		QDS3 & 0.5771 & 0.2339 & 0.1561 & 0.0330
		& 0.7367 & 0.2633 & 1.2633 \\
		QDS4 & 0.0018 & 0.9970 & 0.0012 & $<10^{-4}$
		& 0.0020 & 0.9980 & 1.9980 \\
		\bottomrule
	\end{tabular}
\end{table}

\begin{figure}[!ht]
	\centering
	\includegraphics[width=0.49\textwidth]
	{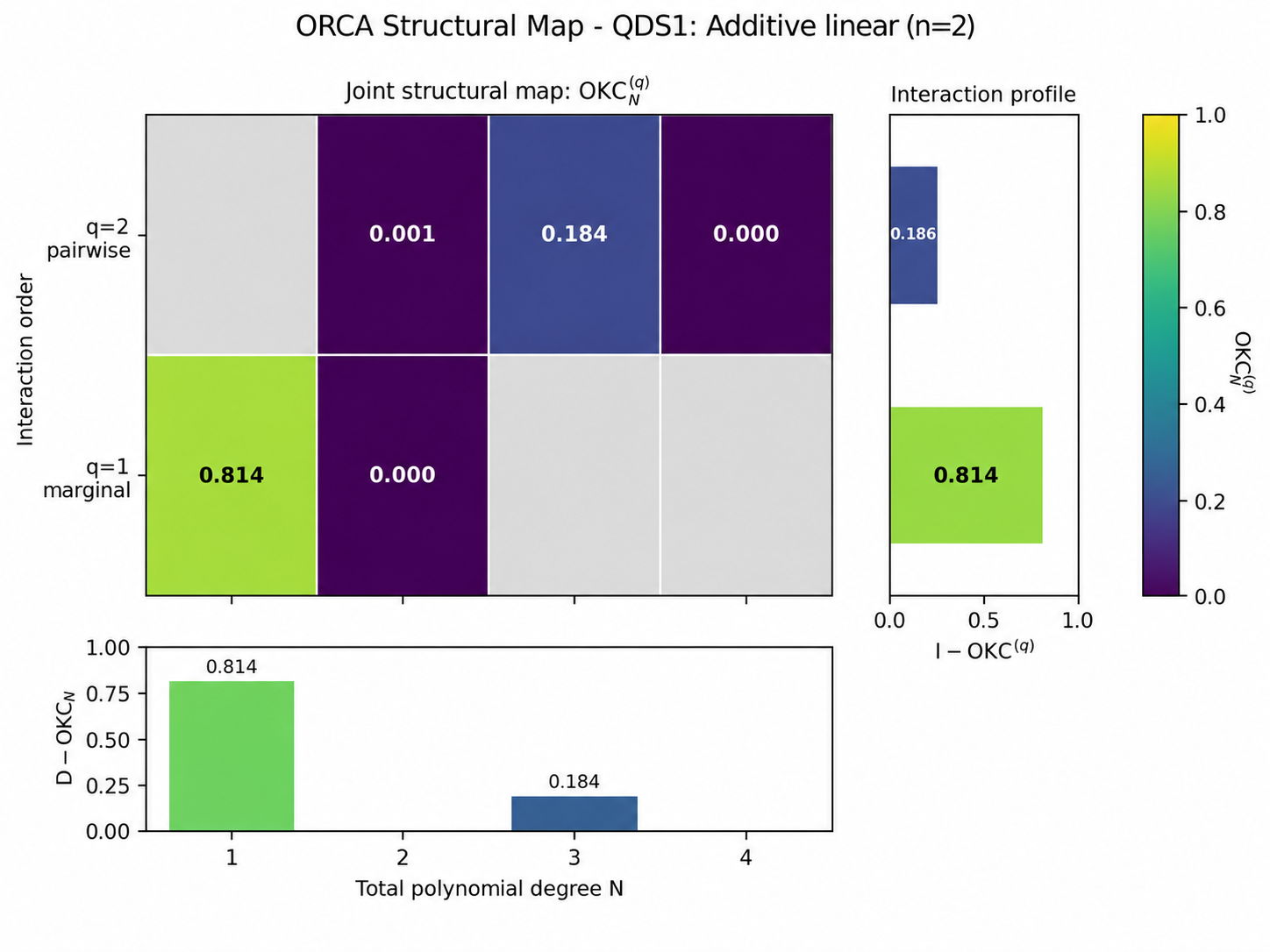}
	\hfill
	\includegraphics[width=0.49\textwidth]
	{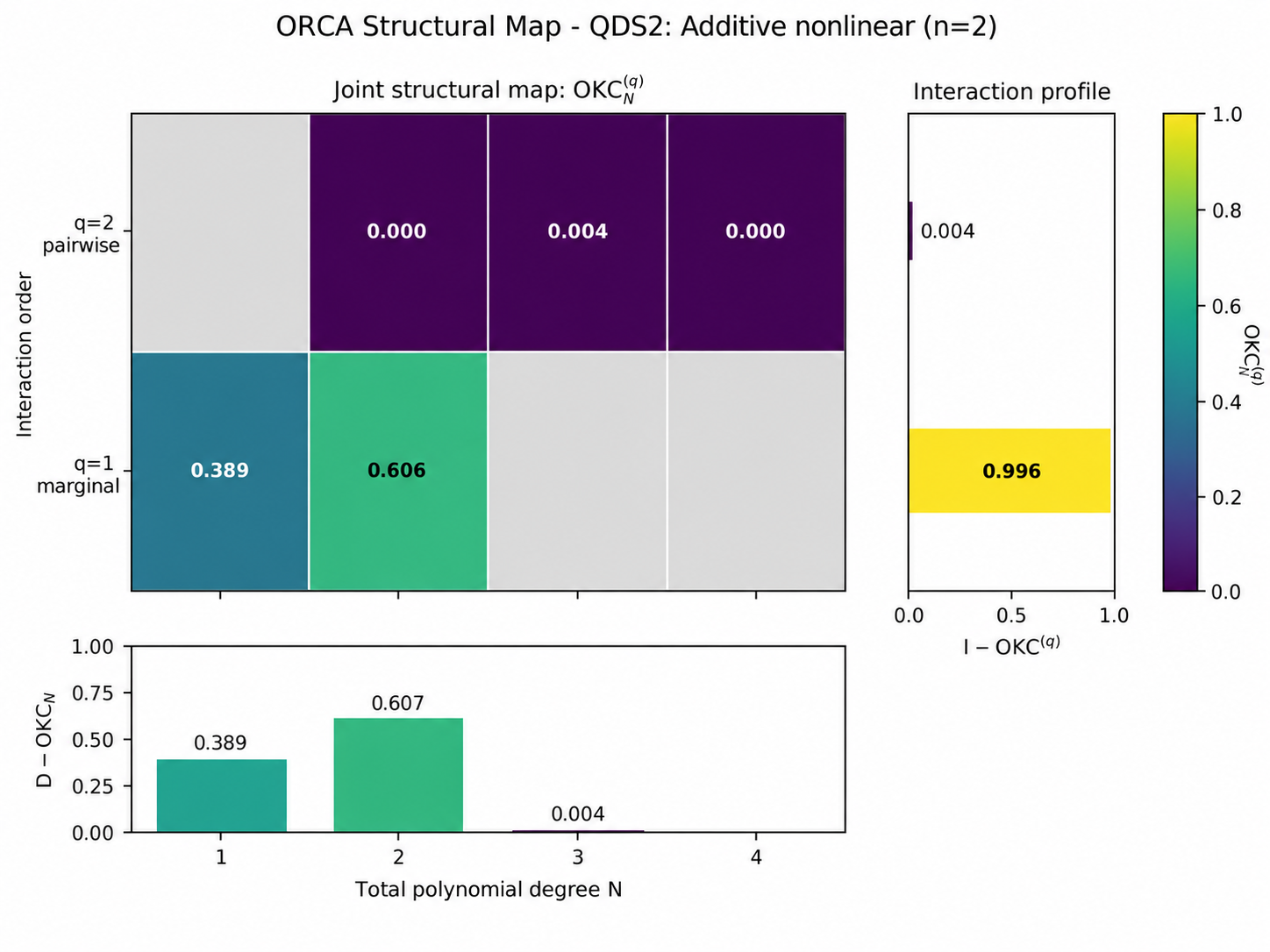}
	
	\includegraphics[width=0.49\textwidth]
	{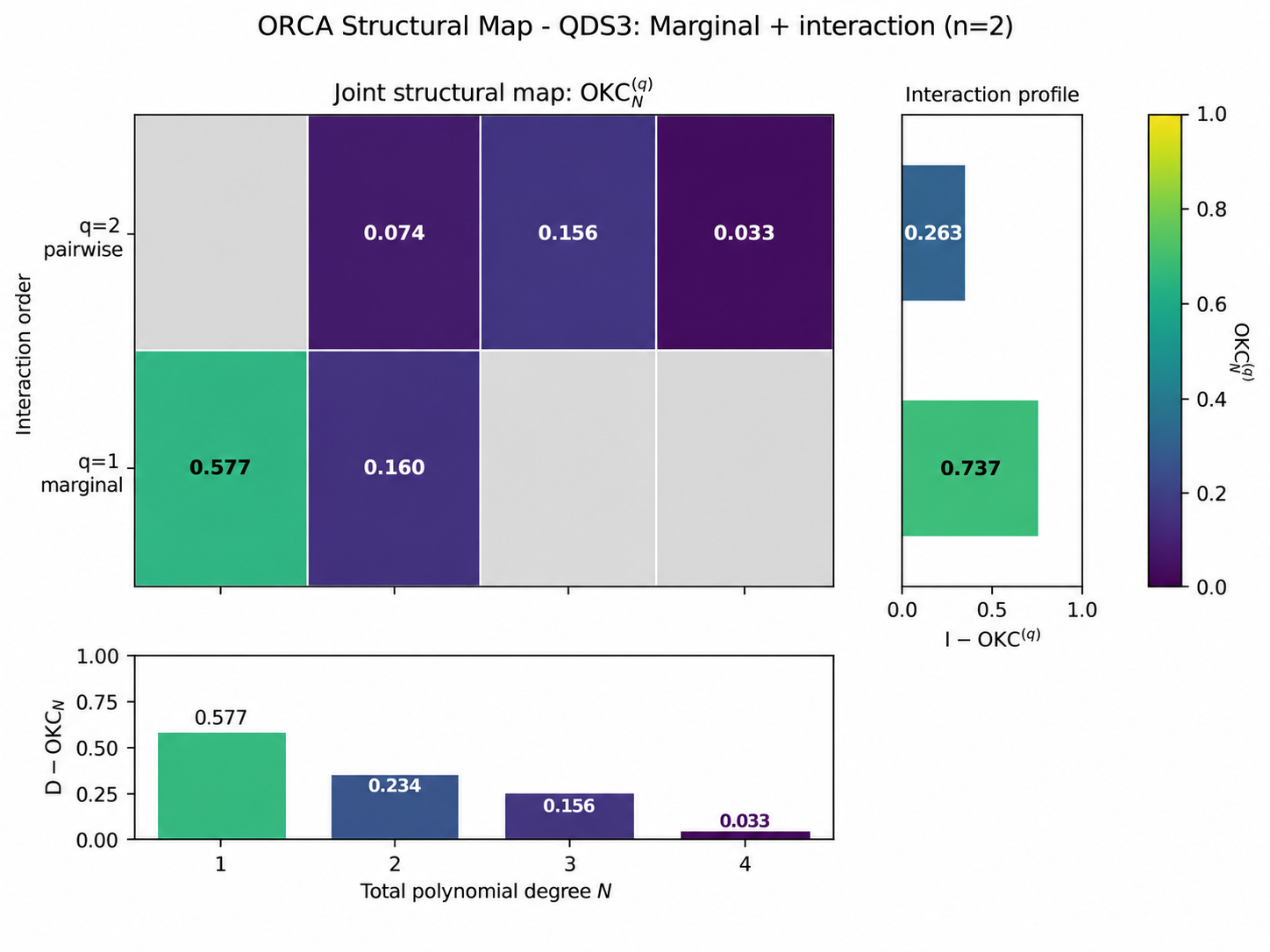}
	\hfill
	\includegraphics[width=0.49\textwidth]
	{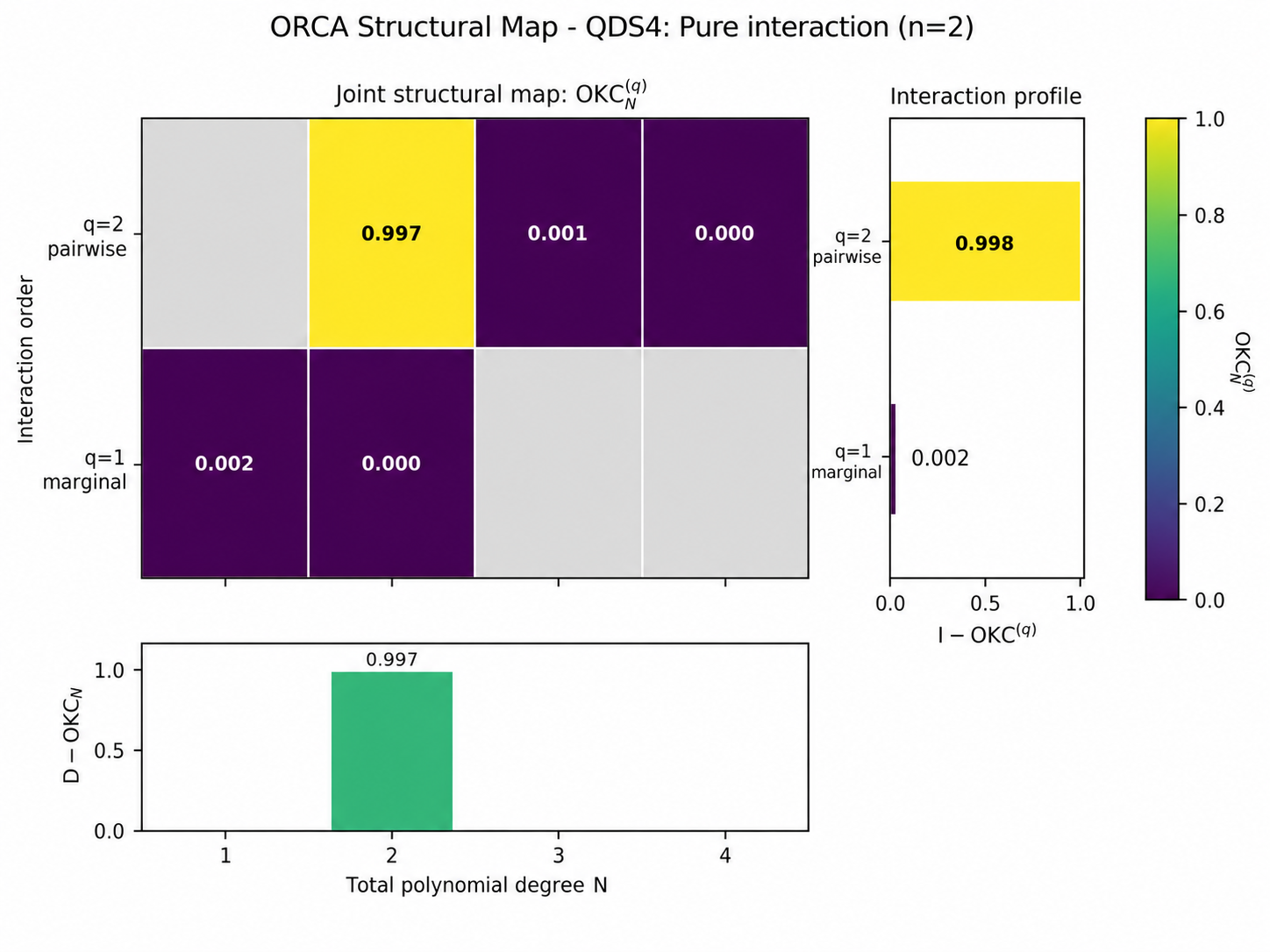}
	\caption{ORCA Structural Maps for QDS1--QDS4.}
	\label{fig:quartet-orca}
\end{figure}

From Table~\ref{tab:quartet-global}, it can be observed that QDS2 is almost
entirely marginal, with $\operatorname{I-OKC}^{(1)}=0.9955$, but its
contribution is predominantly quadratic, since
$\operatorname{D-OKC}_{2}=0.6066$. 
Hence, the fitted classifier is 
nonlinear without relying appreciably on interactions.
QDS4 has the opposite structure. Its contribution is also concentrated at
degree $2$, with $\operatorname{D-OKC}_{2}=0.9970$, but
$\operatorname{I-OKC}^{(2)}=0.9980$. 
Thus, QDS2 and QDS4 have the same
dominant total degree, $N_{\max}=2$, while one fitted classifier is almost
purely marginal and the other is almost purely pairwise interactive. 
This difference is captured by considering polynomial degree and interaction order jointly, as in the indices $\operatorname{OKC}_{N}^{(q)}$ (see Figure~\ref{fig:quartet-orca}).

QDS3 exhibits the expected mixed behavior. Marginal terms account for
$0.7367$ of the squared RKHS norm and pairwise terms for $0.2633$. Its
degree profile is also more distributed, with non-negligible contributions
at all degrees from $1$ to $4$. QDS1 is predominantly marginal, but its
fitted representation is not purely additive since
$\operatorname{I-OKC}^{(2)}=0.1856$. Most of this interaction contribution
occurs at total degree $3$. This result is relevant because QDS1 was generated
from an additive linear expression. ORCA therefore describes the structure
of the fitted SVM rather than simply reproducing the structure used to
generate the training labels.

We next examine the role of the input variables through the indices introduced in
Section~\ref{subsec:variable-okc}. The corresponding graphical output is shown in
Figure~\ref{fig:quartet-orca-var-input}, which presents the
ORCA Variable Mechanism Panels for QDS1--QDS4. These panels compare the
marginal and global contributions of each input variable and identify the
corresponding pairwise contribution.
Specifically, 
for QDS2,
\[\operatorname{M-OKC}_{1}=0.6026,
\qquad
\operatorname{M-OKC}_{2}=0.3929,
\]
while
\[
\operatorname{G-OKC}_{1}=0.6071,
\qquad
\operatorname{G-OKC}_{2}=0.3974.
\]
The small difference between the marginal and global indices is consistent
with the negligible pairwise contribution
$\operatorname{P-OKC}_{12}=0.0045$. For QDS4, by contrast,
\[
\operatorname{M-OKC}_{1}=0.0015,
\qquad
\operatorname{M-OKC}_{2}=0.0005,
\]
whereas
\[
\operatorname{G-OKC}_{1}=0.9995,
\qquad
\operatorname{G-OKC}_{2}=0.9985.
\]
Thus, both input variables exhibit substantial global contributions to the fitted QDS4 classifier, but almost
all their contribution occurs through their joint participation, as also
shown by $\operatorname{P-OKC}_{12}=0.9980$. QDS1 and QDS3 occupy
intermediate positions, with pairwise contributions $0.1856$ and $0.2633$,
respectively.

\begin{figure}[!ht]
	\centering
	\includegraphics[width=0.49\textwidth]
	{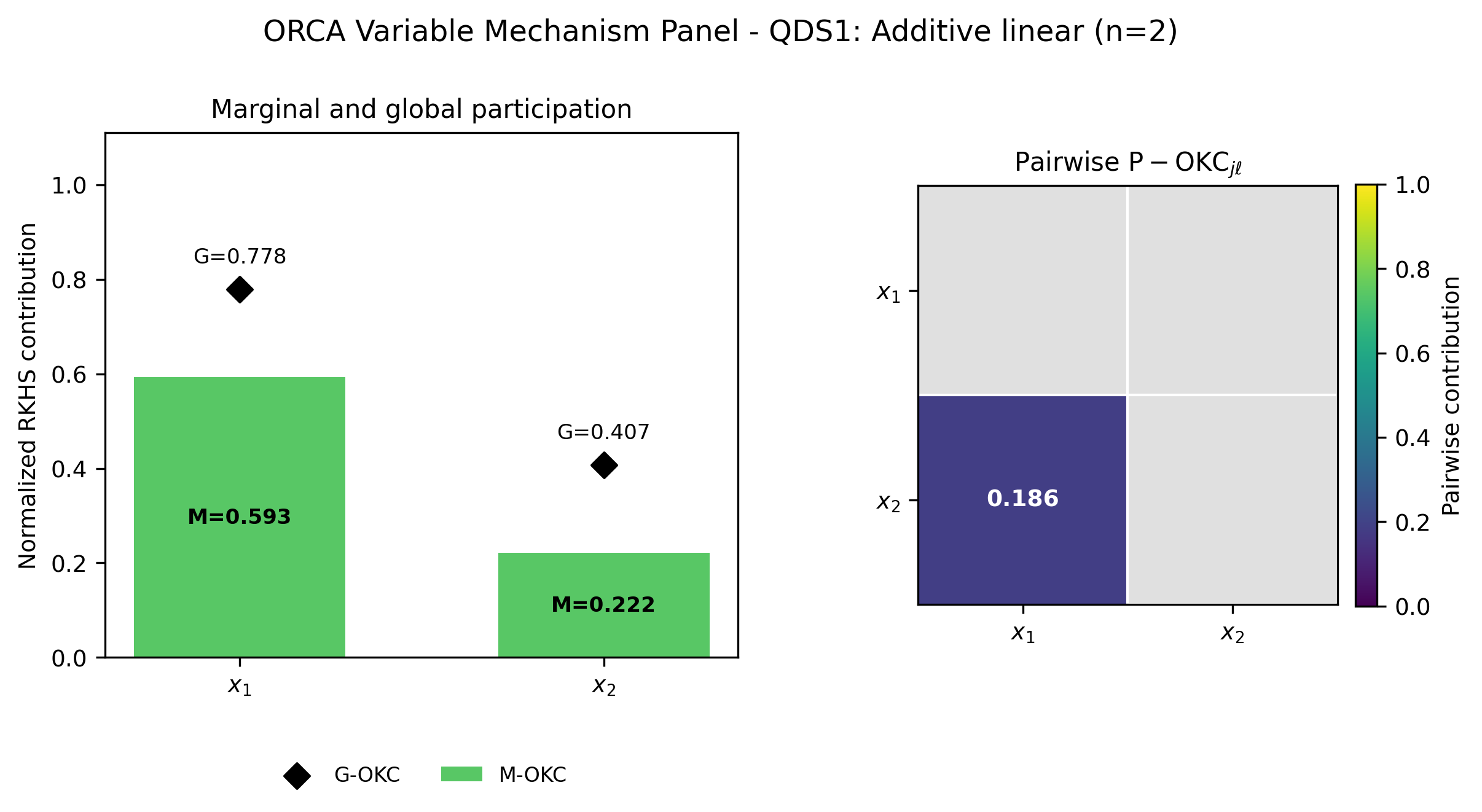}
	\hfill
	\includegraphics[width=0.49\textwidth]
	{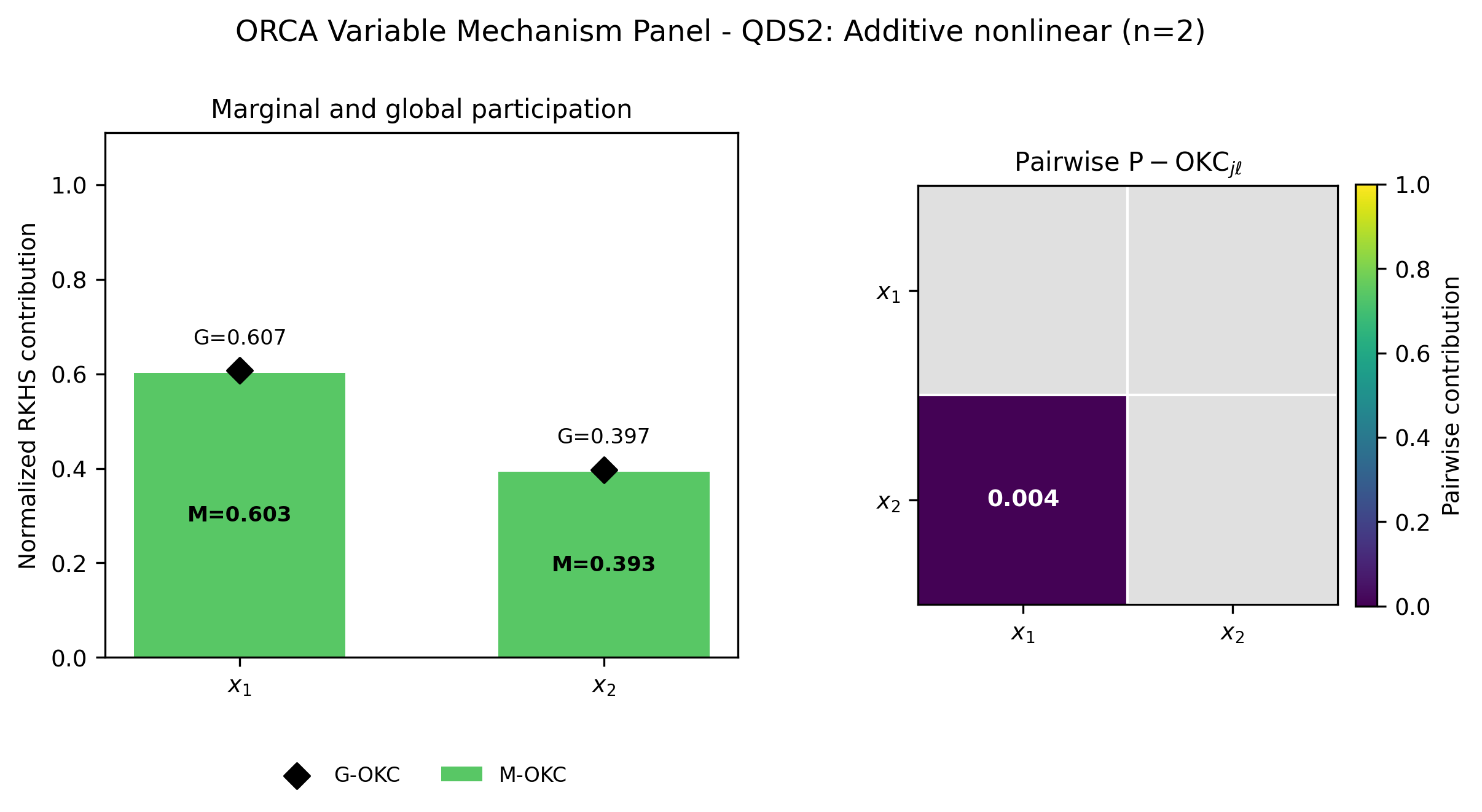}
	
	\includegraphics[width=0.49\textwidth]
	{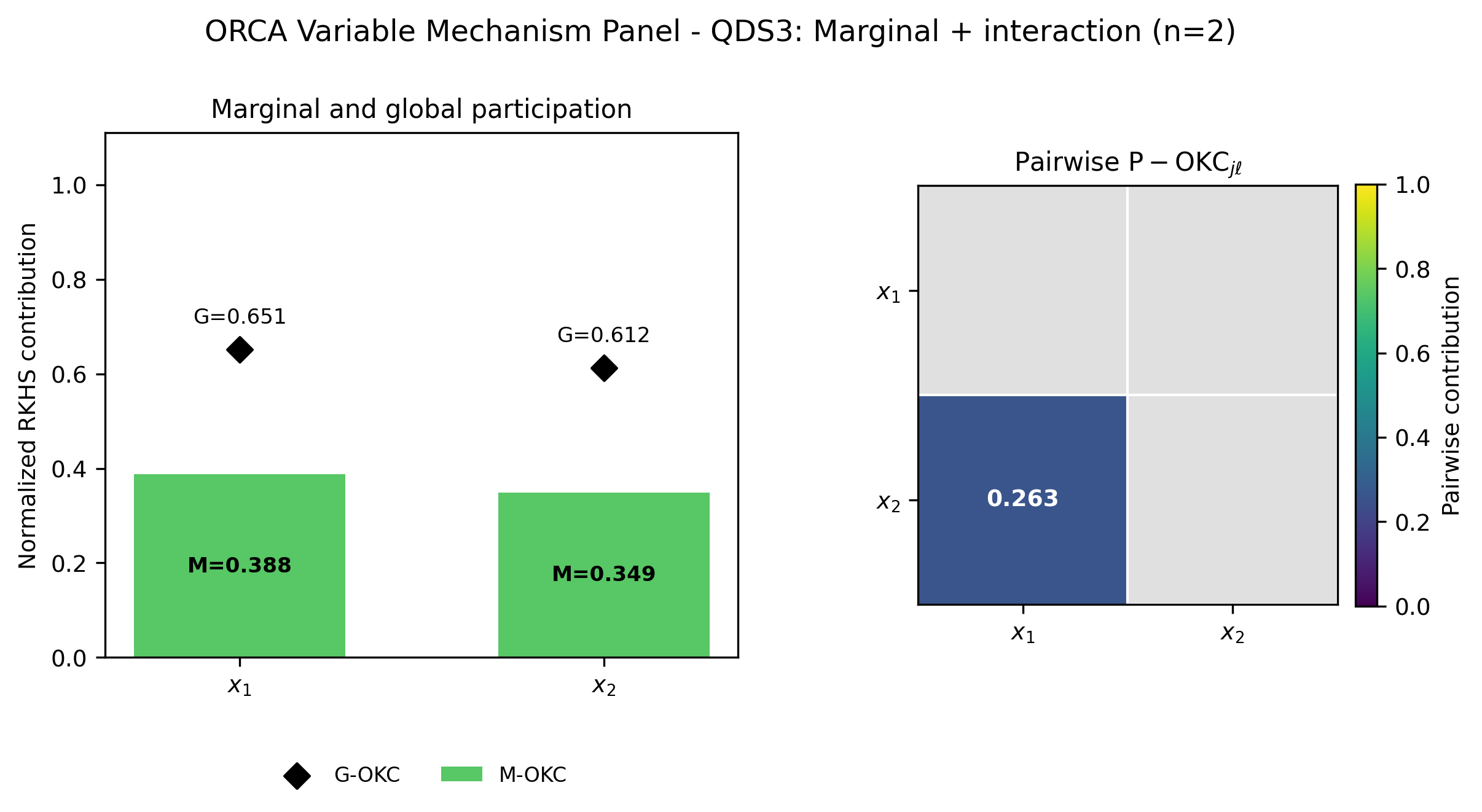}
	\hfill
	\includegraphics[width=0.49\textwidth]
	{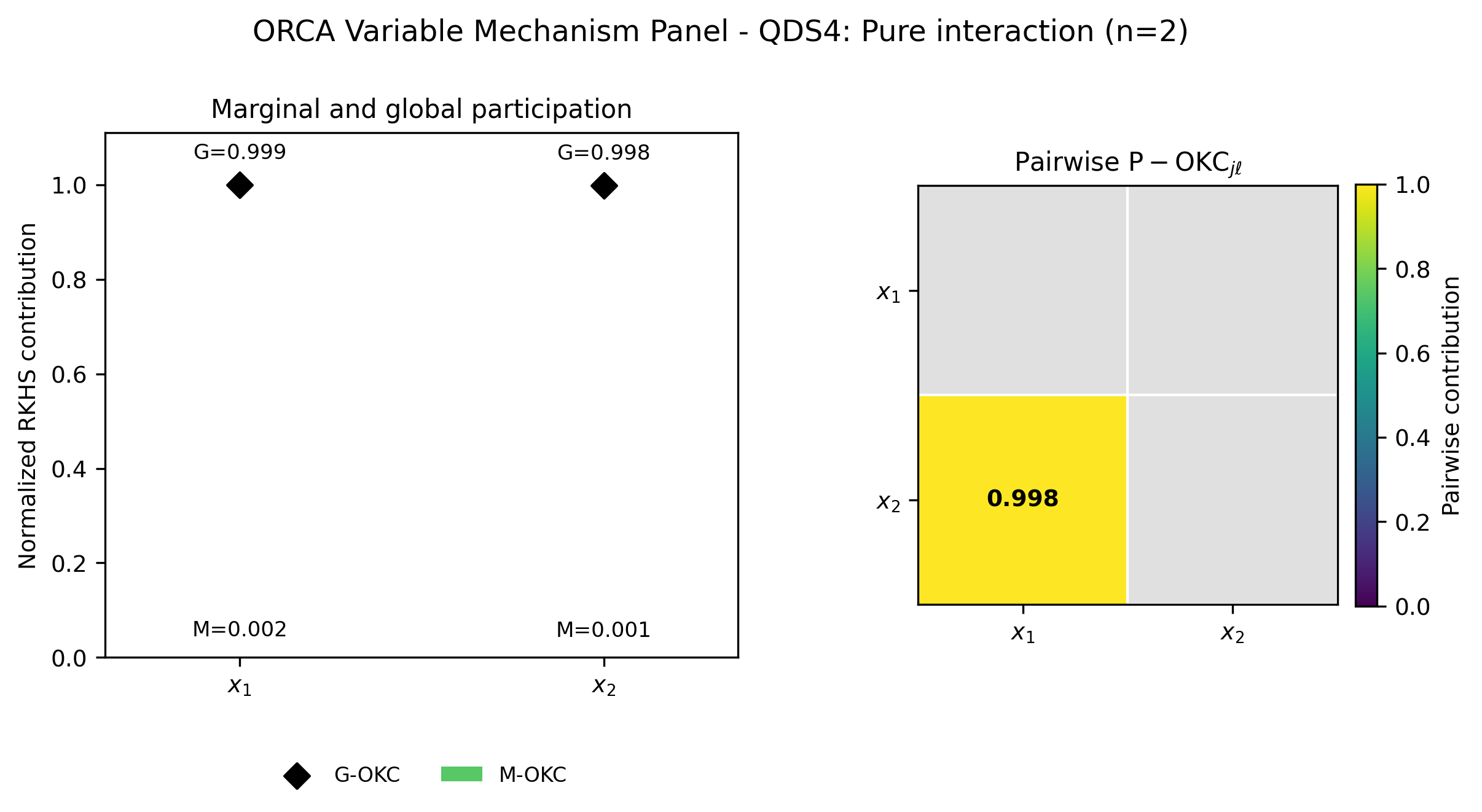}
	\caption{ORCA Variable Mechanism Panels for QDS1--QDS4.}
	\label{fig:quartet-orca-var-input}
\end{figure}

Ultimately, the ORCA Quartet shows that the degree, interaction, and input-variable
structures of fitted SVMs can differ substantially even when the kernel, the
regularization parameter, and the predictive performance are kept fixed.

\subsection{Spiral classification problem}
\label{sec:spiral}

This section examines the double-spiral data set through a controlled
comparison in which the data are kept fixed and only the
orthogonal-polynomial kernel is changed. The training and test samples, the
regularization parameter $C$, the input dimension $d$, and the truncation
level $n$ are held fixed throughout the comparison.

Under this setting, differences between the fitted classifiers can be
attributed to the choice of orthogonal-polynomial kernel rather than to
changes in the data or in the model dimension. ORCA is then used to compare
how the different kernels affect the internal structure of the fitted SVM in
terms of polynomial degree, interaction order, and input-variable
participation.

The data set consists of $300$ points in $\mathbb{R}^2$ arranged along two
interleaved spirals, with $150$ points per class and class labels
$y_i\in\{-1,+1\}$. The spirals complete approximately two turns, producing a
nonlinear classification problem with a curved decision boundary. The
observations are split by stratified random sampling into a training sample
of $m=200$ points, with $100$ points from each class, and a test sample of
$100$ points, with $50$ points from each class.

Each input variable is rescaled linearly to $[-1,1]$ before training, so that
the transformed data lie in the product domain $[-1,1]^d$. Specifically, for
the $j$-th input variable of the $i$-th training point,
\begin{equation}\label{eq:scaling}
  \widetilde{x}_{ij}
  =2\,\frac{x_{ij}-\min\limits_{1\le r\le m}x_{rj}}
           {\max\limits_{1\le r\le m}x_{rj}-\min\limits_{1\le r\le m}x_{rj}}-1,
\end{equation}
where $i=1,\dots,m$ and $j=1,\dots,d$. This transformation places the inputs
in the natural domain of the Jacobi polynomials, with orthonormality defined
with respect to the product measure $\mu^{\otimes d}$ on $[-1,1]^d$. The test
points are transformed with the same map, that is, with the minima and maxima
of the training sample.

Before comparing different orthogonal-polynomial kernels, we first examine
the effect of the truncation level $n$ on the fitted decision boundary. For
this purpose, we consider the Legendre case, $\alpha=\beta=0$, with $C=1$.
Figure~\ref{fig:spiral_truncation} shows the decision boundary
$g(\mathbf{x})=0$ for $n\in\{1,2,3,5,8,12,14,16\}$.

The sequence illustrates how the truncation level affects the geometry of the
fitted classifier through the dimension $(n+1)^2$ of $\mathcal{H}^{(2)}_n$.
For $n=1$, the fitted boundary is approximately linear. For $n=2$ and $n=3$,
the boundary becomes progressively curved. At $n=5$, it begins to follow the
main structure of the spiral arms, while for $n=8$ it traces a clear
double-spiral contour. For $n\in\{12,14,16\}$, the overall spiral shape is
maintained, but additional local oscillations appear.

Based on this comparison, we set $n=8$ for the remainder of the experiment.
At this truncation level, the fitted boundary already captures the main
double-spiral structure, while larger values mainly introduce additional
local variation. Fixing $n=8$ also ensures that all kernels are compared
within feature spaces of the same dimension,
$\dim(\mathcal{H}^{(2)}_8)=81$. From this point onward, the data, training
and test samples, $C$, $d$, and $n$ are kept fixed, and only the
orthogonal-polynomial kernel is varied.

\begin{figure}[htbp]
  \centering
  \includegraphics[width=\textwidth]{figures/spiral_boundaries_varying_n.pdf}
  \caption{Decision boundaries $g(\mathbf{x})=0$ for the SVM with the Jacobi
  tensor-product kernel on the double-spiral data set, with $\alpha=\beta=0$
  (Legendre case) and $C=1$, for different truncation levels $n$.
  Top row, left to right: $n=1,2,3,5$. Bottom row, left to right:
  $n=8,12,14,16$. The background colors represent the values of the decision
  function $g(\mathbf{x})$; the bold curve is the zero level set; points are
  colored according to their class labels.}
  \label{fig:spiral_truncation}
\end{figure}

We next compare five Jacobi tensor-product kernels while keeping the data,
training and test samples, $C$, $d$, and $n=8$ fixed. The five cases are: the
Legendre kernel, corresponding to $(\alpha,\beta)=(0,0)$; the Chebyshev kernel,
associated with the Chebyshev polynomials of the first kind,
$(\alpha,\beta)=(-\tfrac{1}{2},-\tfrac{1}{2})$; the
Gegenbauer kernel with $\lambda=\tfrac{5}{2}$, that is,
$(\alpha,\beta)=(2,2)$; an asymmetric Jacobi kernel with
$(\alpha,\beta)=(3,1)$; and the reversed asymmetric kernel obtained by
interchanging the two parameters, $(\alpha,\beta)=(1,3)$.

This choice allows us to compare the Legendre case with the Chebyshev and
Gegenbauer systems and, at the same time, to examine the effect of introducing
asymmetry in the Jacobi weight. The pair $(\alpha,\beta)$ and its reversed
version $(\beta,\alpha)$ provide a direct comparison of the effect of changing
the orientation of this asymmetry while preserving the same two parameter
values.

As for the ORCA Quartet, $C=1$ and the numerical tolerance is $10^{-9}$. All
the conservation identities of Step~3 hold, and the orthogonal expansion
reproduces the decision values at the test points, with absolute errors below
$10^{-12}$; the same holds for the real data set of
Section~\ref{sec:realdata}.

The ORCA results are summarized in two tables. The first table reports the
interaction structure and the participation of the two input variables. For
$d=2$, the mean interaction order and the pairwise contribution can be
recovered directly from these quantities, since
\[
  \bar q=\operatorname{I-OKC}^{(1)}+2\operatorname{I-OKC}^{(2)}
        =1+\operatorname{I-OKC}^{(2)},
  \qquad
  \operatorname{P-OKC}_{12}=\operatorname{I-OKC}^{(2)}.
\]

\begin{table}[htbp]
\centering
\small
\setlength{\tabcolsep}{4pt}
\caption{Interaction and input-variable structure for the double-spiral data set with $n=8$.}
\label{tab:spiral_interaction}
\begin{tabular}{lccccccc}
\toprule
$(\alpha,\beta)$ & Test accuracy & $\operatorname{I-OKC}^{(1)}$ & $\operatorname{I-OKC}^{(2)}$ & $\operatorname{M-OKC}_1$ & $\operatorname{G-OKC}_1$ & $\operatorname{M-OKC}_2$ & $\operatorname{G-OKC}_2$ \\
\midrule
$(0,0)$ & 0.970 & 0.1956 & 0.8044 & 0.0929 & 0.8973 & 0.1027 & 0.9071 \\
$(-\tfrac{1}{2},-\tfrac{1}{2})$ & 0.970 & 0.1281 & 0.8719 & 0.0829 & 0.9548 & 0.0452 & 0.9171 \\
$(2,2)$ & 0.990 & 0.2985 & 0.7015 & 0.1056 & 0.8070 & 0.1930 & 0.8944 \\
$(3,1)$ & 0.970 & 0.1651 & 0.8349 & 0.0693 & 0.9042 & 0.0958 & 0.9307 \\
$(1,3)$ & 0.970 & 0.1688 & 0.8312 & 0.0607 & 0.8919 & 0.1081 & 0.9393 \\
\bottomrule
\end{tabular}
\end{table}

The second table summarizes the distribution of the OKC indices across total
polynomial degrees. In addition to the dominant and cumulative degree
summaries, we report the total contributions associated with even and odd
values of the total degree $N$:
\[
  D_{\mathrm{even}}=\sum_{\substack{0\le N\le dn\\ N\ \text{even}}}\operatorname{D-OKC}_N,
  \qquad
  D_{\mathrm{odd}}=\sum_{\substack{0\le N\le dn\\ N\ \text{odd}}}\operatorname{D-OKC}_N .
\]
Hence, $D_{\mathrm{even}}+D_{\mathrm{odd}}=1$. For the present experiment,
$d=2$ and $n=8$, so that $N=0,\dots,16$.

\begin{table}[htbp]
\centering
\caption{Total-degree structure for the double-spiral data set with $n=8$.}
\label{tab:spiral_degree}
\begin{tabular}{lcccccc}
\toprule
$(\alpha,\beta)$ & $D_{\mathrm{even}}$ & $D_{\mathrm{odd}}$ & $N_{\max}$ & $N_{0.90}$ & $N_{0.95}$ & $N_{0.99}$ \\
\midrule
$(0,0)$ & 0.0088 & 0.9912 & 9 & 15 & 15 & 15 \\
$(-\tfrac{1}{2},-\tfrac{1}{2})$ & 0.0120 & 0.9880 & 7 & 15 & 15 & 15 \\
$(2,2)$ & 0.0073 & 0.9927 & 7 & 13 & 13 & 15 \\
$(3,1)$ & 0.3884 & 0.6116 & 9 & 13 & 13 & 15 \\
$(1,3)$ & 0.3944 & 0.6056 & 9 & 13 & 14 & 15 \\
\bottomrule
\end{tabular}
\end{table}

Figures~\ref{fig:spiral_structural_maps}, \ref{fig:spiral_variable_panels}
and~\ref{fig:spiral_radar} show the ORCA Structural Maps, the ORCA Variable
Mechanism Panels and the ORCA comparative radar of the five fitted
classifiers.

\begin{figure}[htbp]
  \centering
  \includegraphics[width=0.49\textwidth]{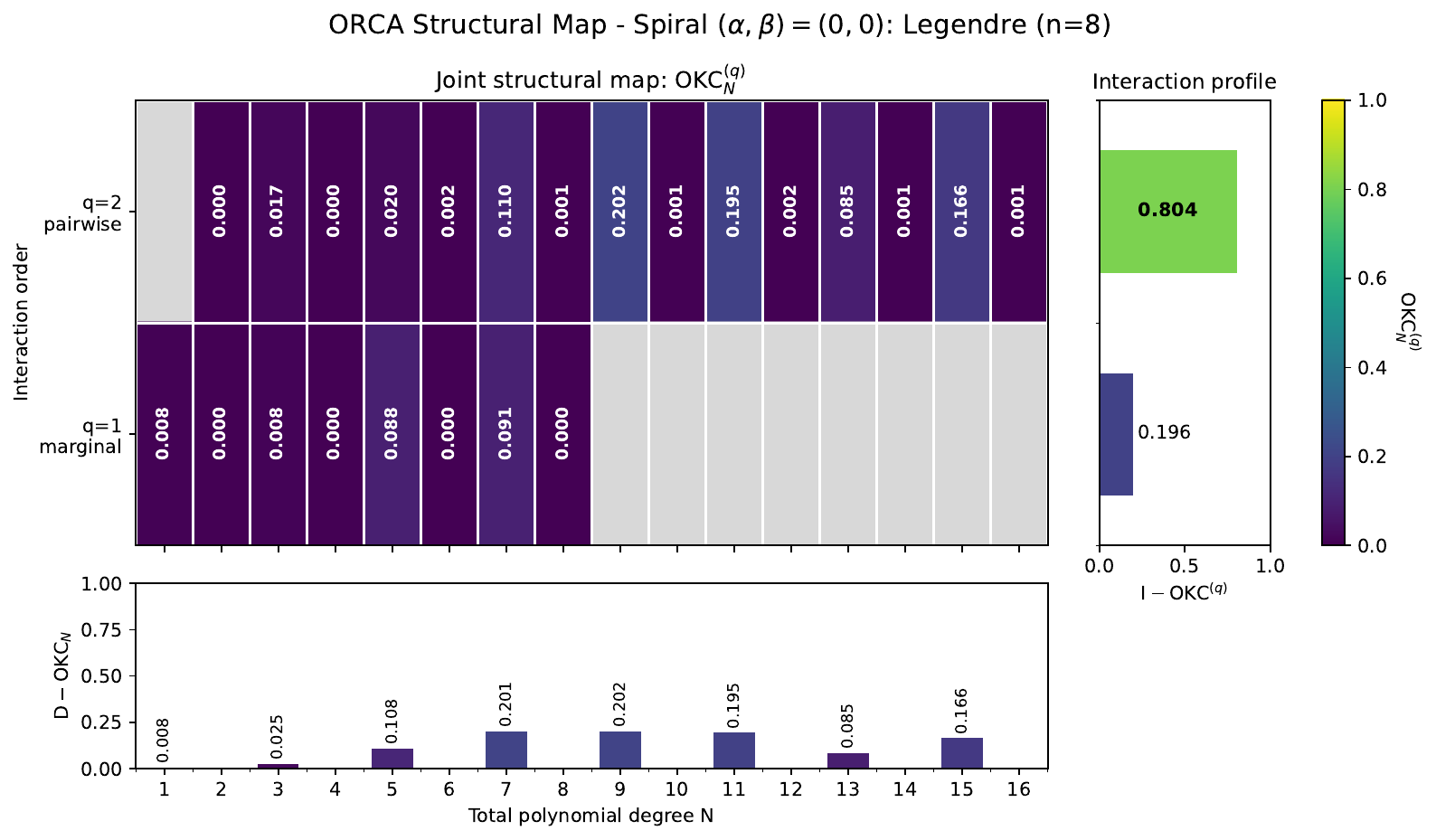}\\[2mm]
  \includegraphics[width=0.49\textwidth]{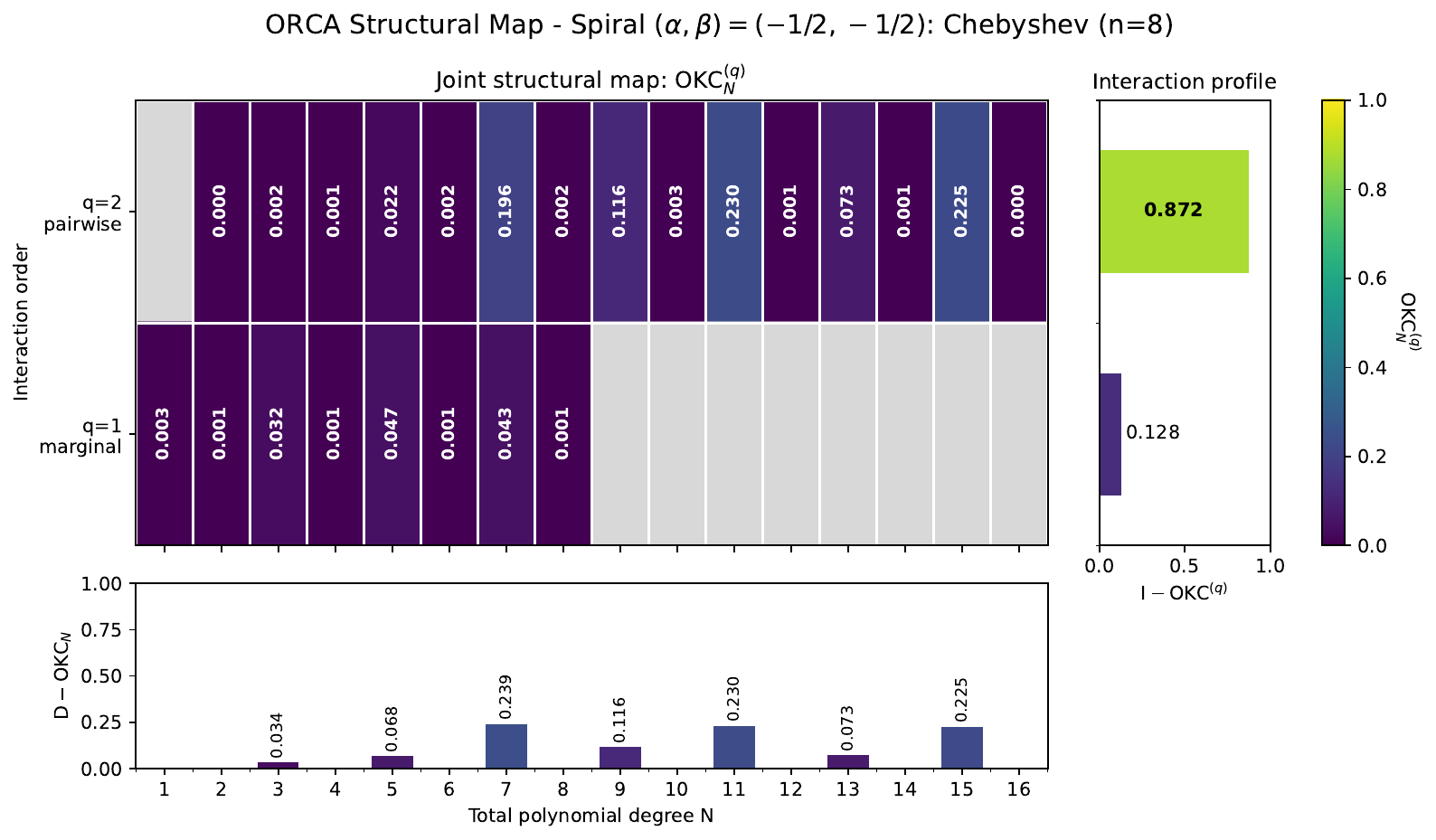}\hfill
  \includegraphics[width=0.49\textwidth]{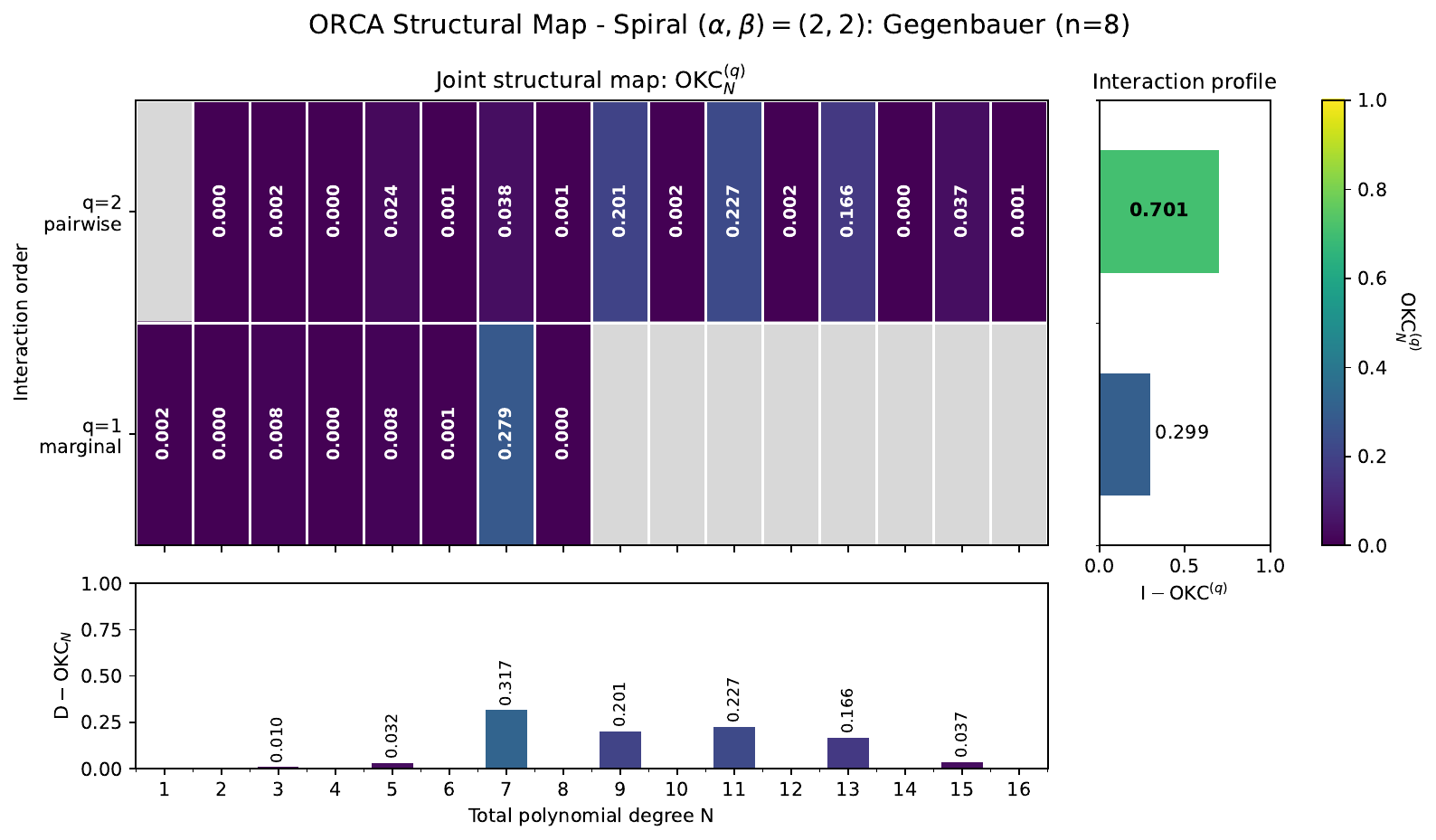}\\[2mm]
  \includegraphics[width=0.49\textwidth]{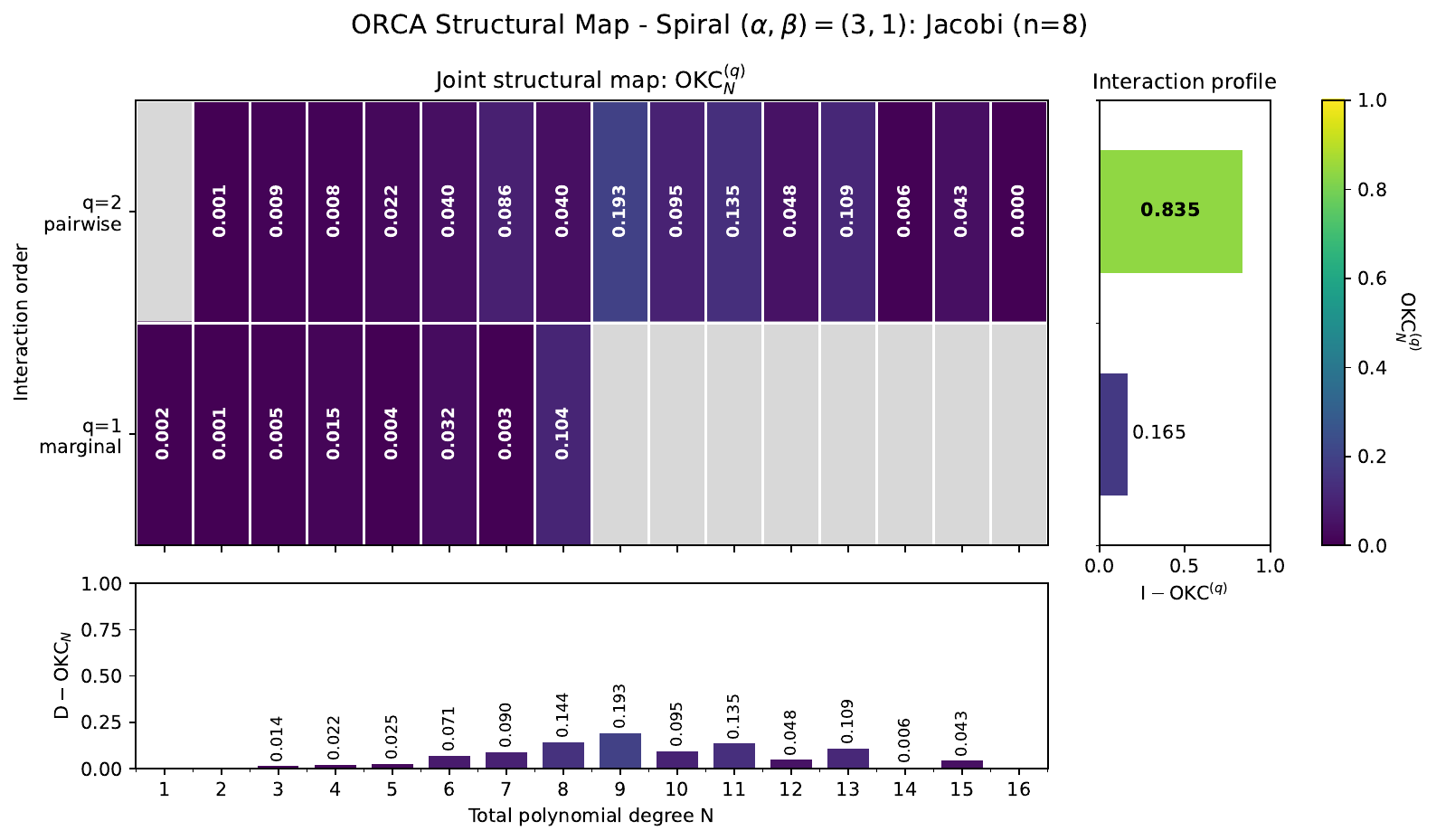}\hfill
  \includegraphics[width=0.49\textwidth]{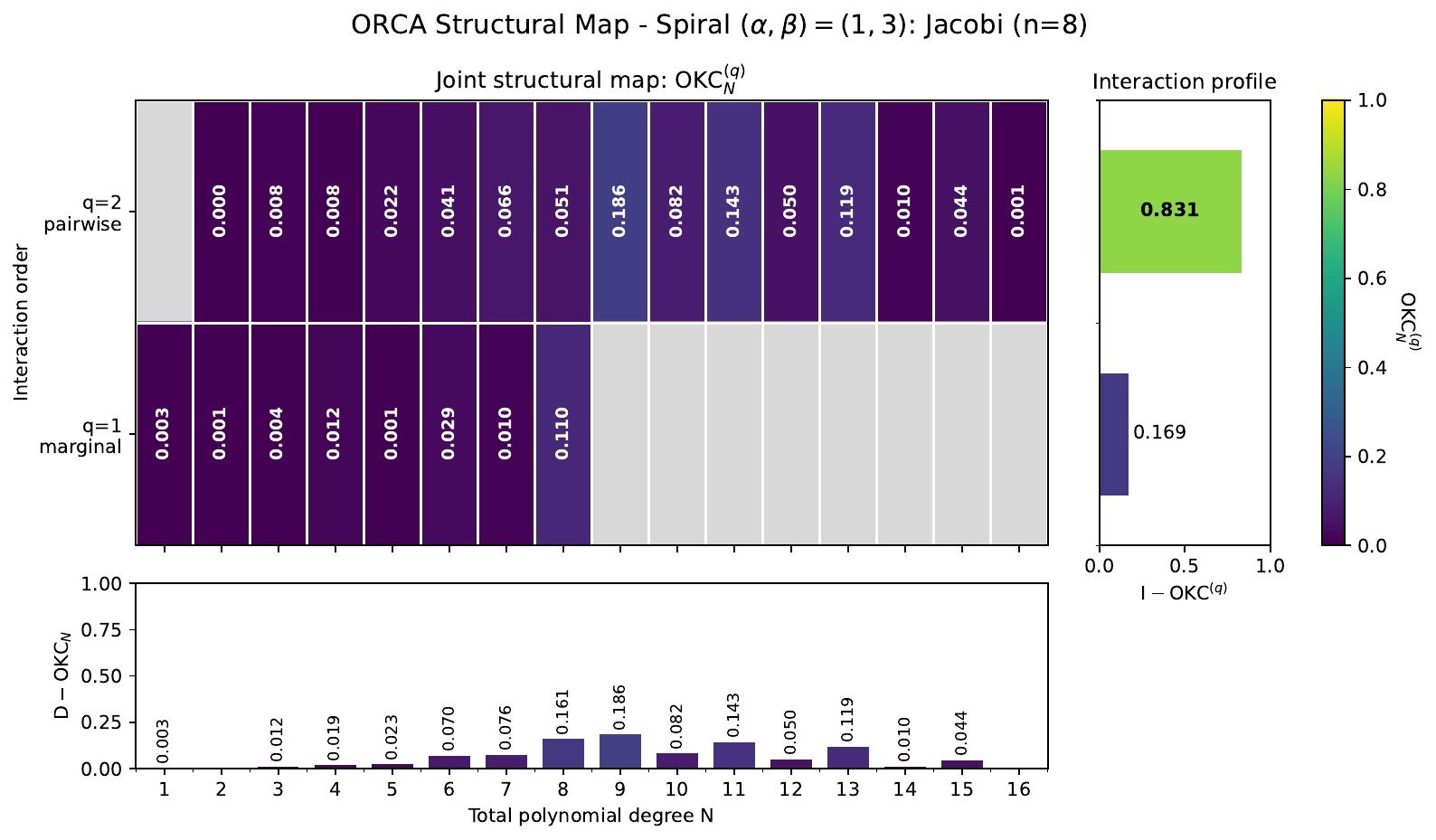}
  \caption{ORCA Structural Maps for the double-spiral data set with $n=8$ and
  $C=1$. Top: $(\alpha,\beta)=(0,0)$ (Legendre). Middle row:
  $(-\tfrac{1}{2},-\tfrac{1}{2})$ (Chebyshev) and $(2,2)$ (Gegenbauer,
  $\lambda=\tfrac{5}{2}$). Bottom row: $(3,1)$ and $(1,3)$ (asymmetric
  Jacobi).}
  \label{fig:spiral_structural_maps}
\end{figure}

\begin{figure}[htbp]
  \centering
  \includegraphics[width=0.49\textwidth]{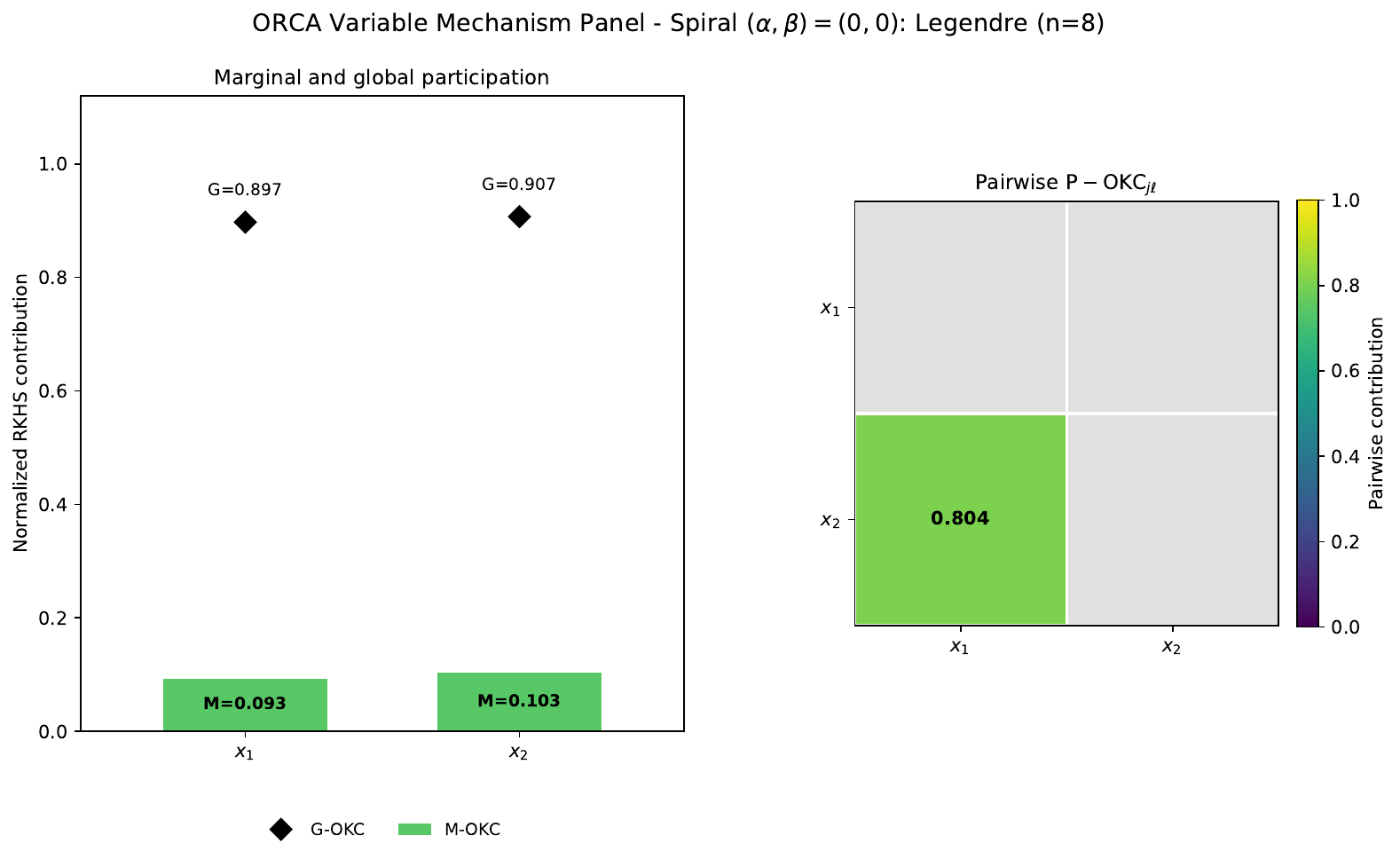}\\[2mm]
  \includegraphics[width=0.49\textwidth]{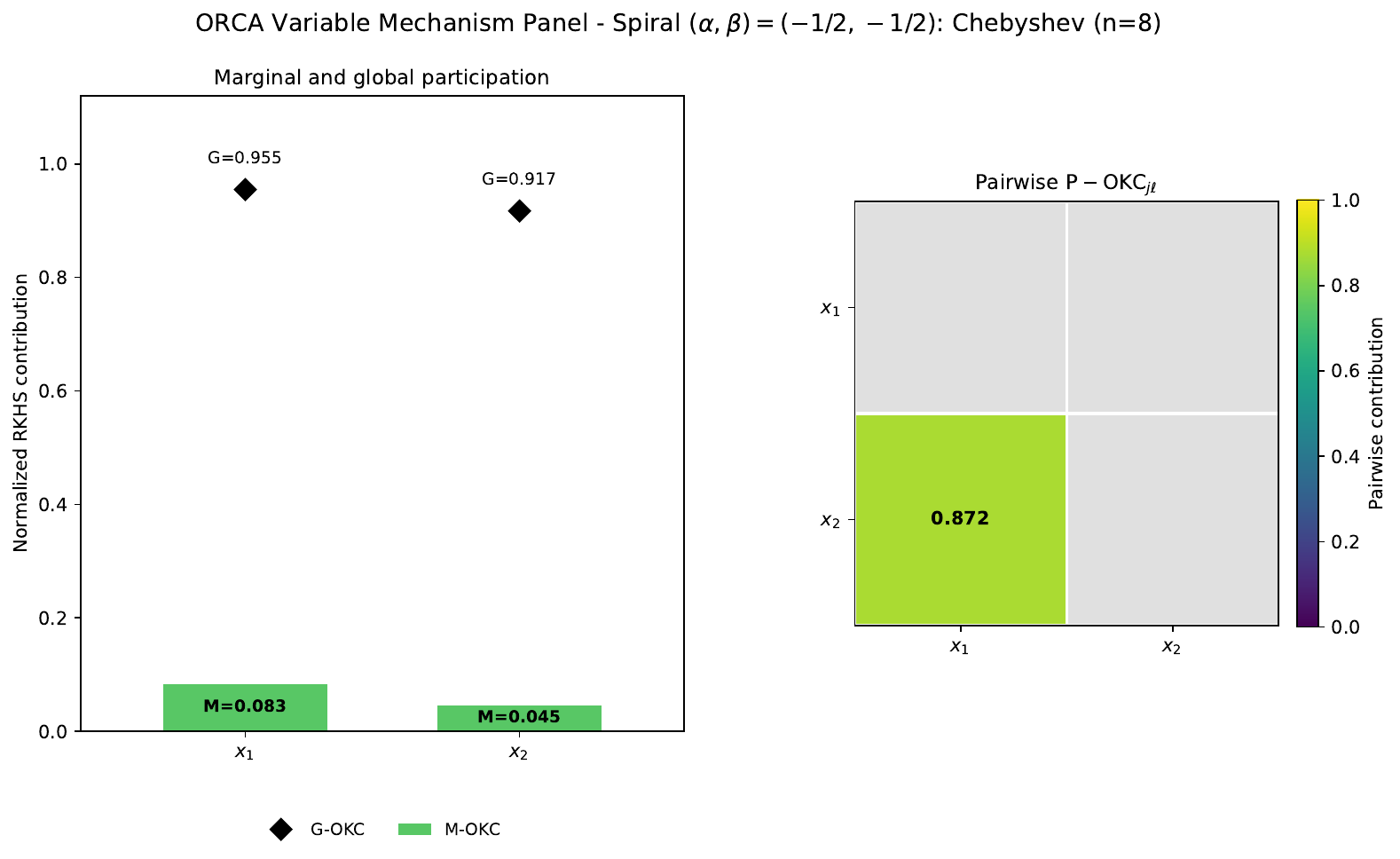}\hfill
  \includegraphics[width=0.49\textwidth]{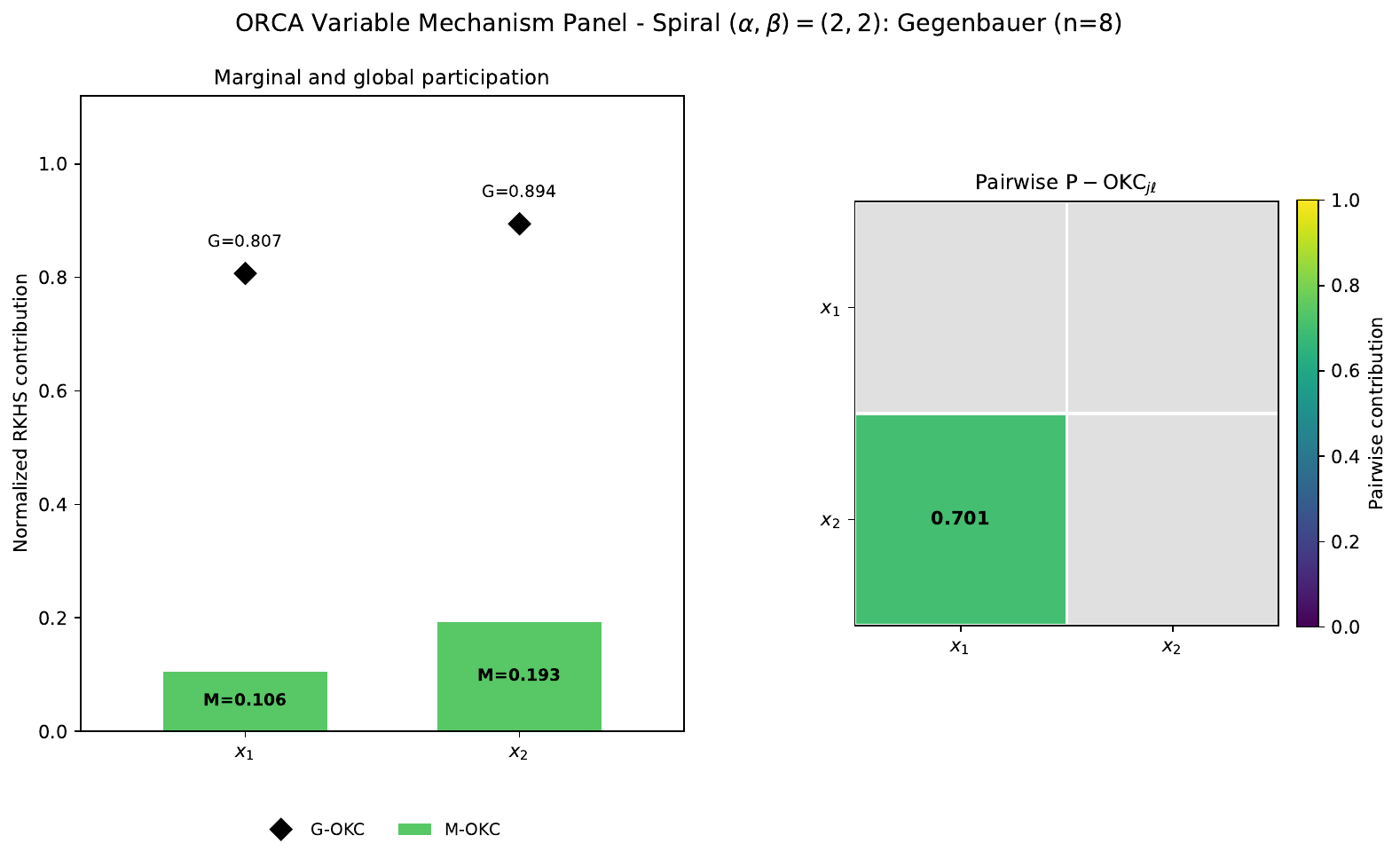}\\[2mm]
  \includegraphics[width=0.49\textwidth]{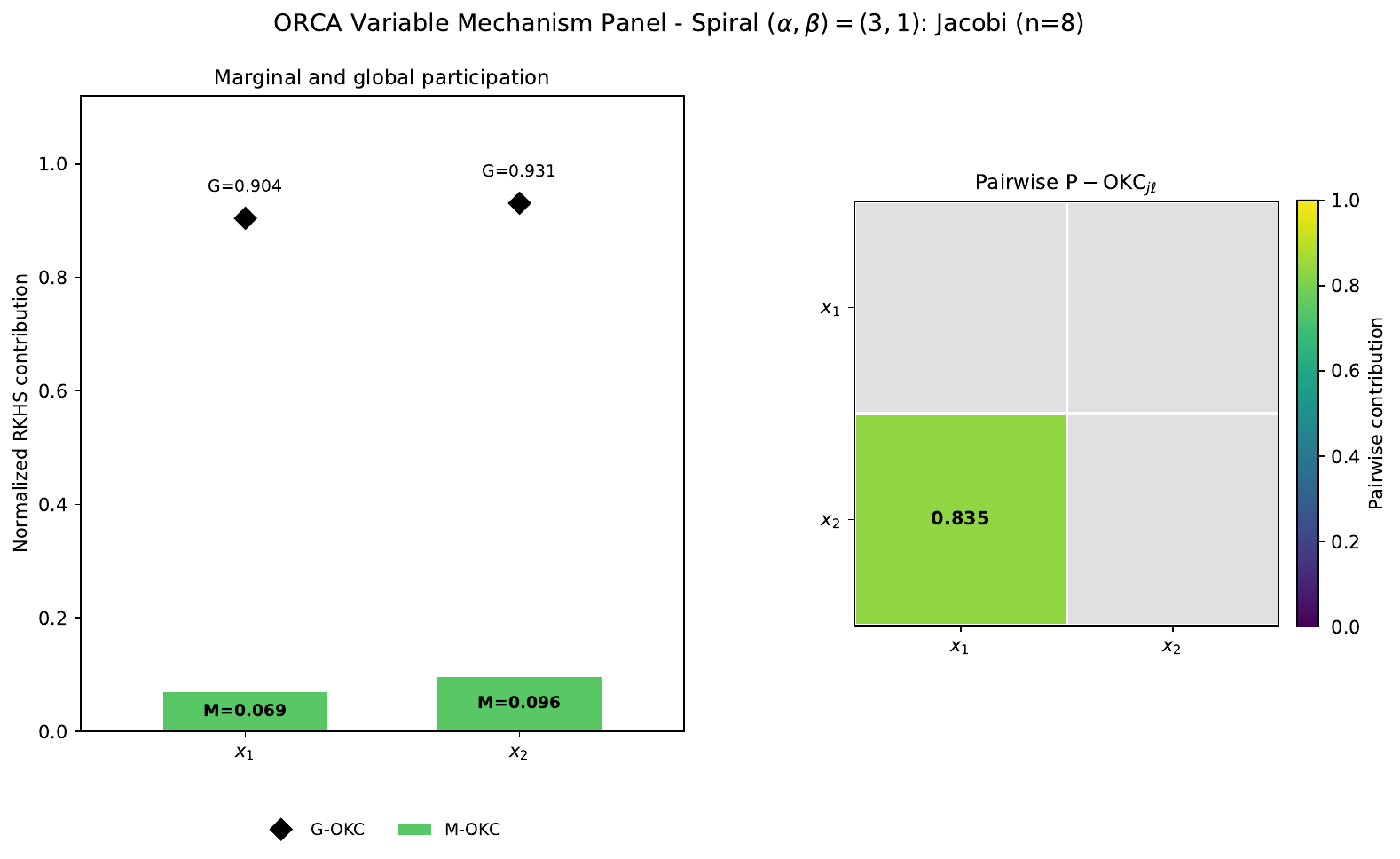}\hfill
  \includegraphics[width=0.49\textwidth]{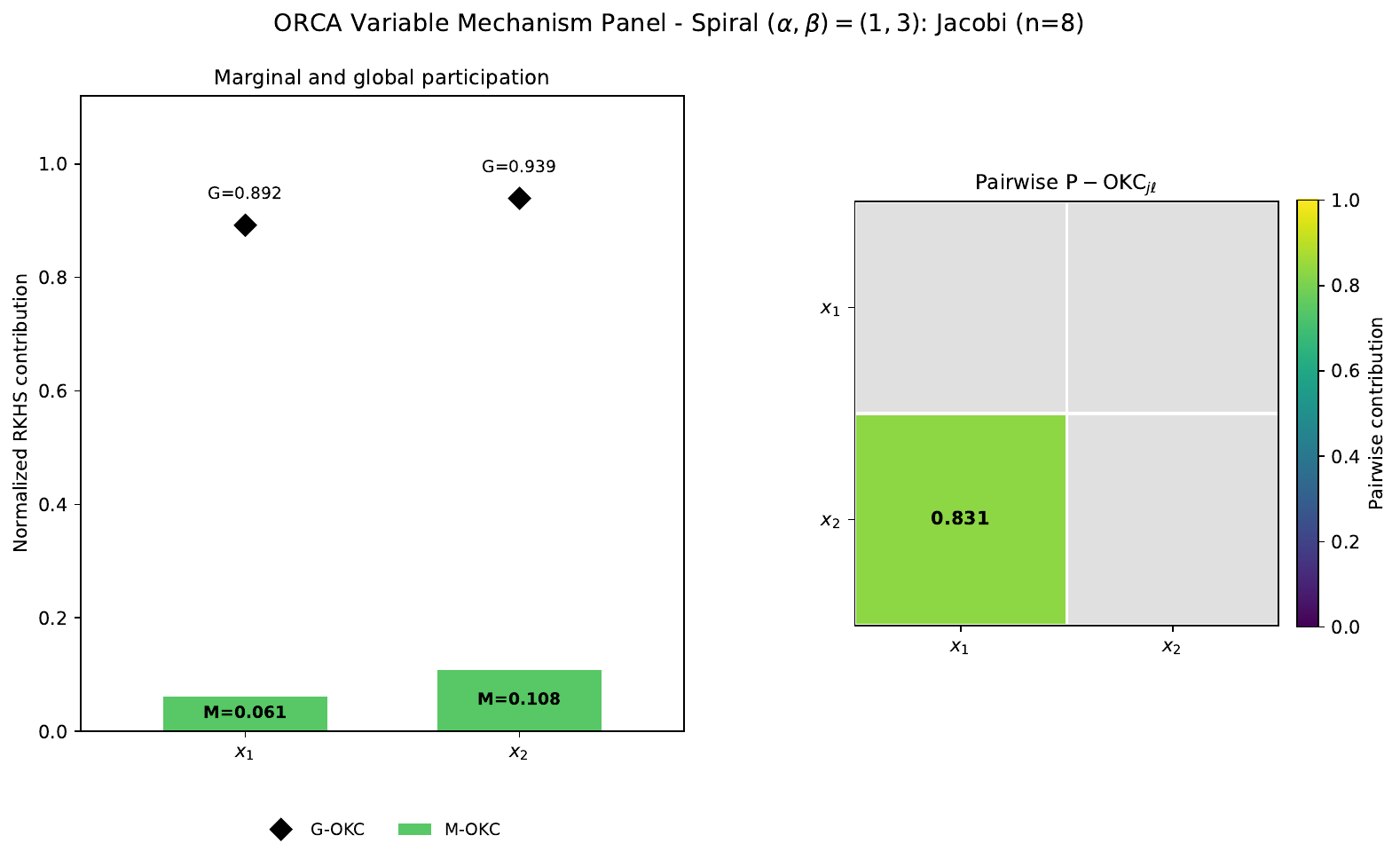}
  \caption{ORCA Variable Mechanism Panels for the double-spiral data set with
  $n=8$ and $C=1$. The panels are arranged as in
  Figure~\ref{fig:spiral_structural_maps}.}
  \label{fig:spiral_variable_panels}
\end{figure}

\begin{figure}[htbp]
  \centering
  \includegraphics[width=0.5\textwidth]{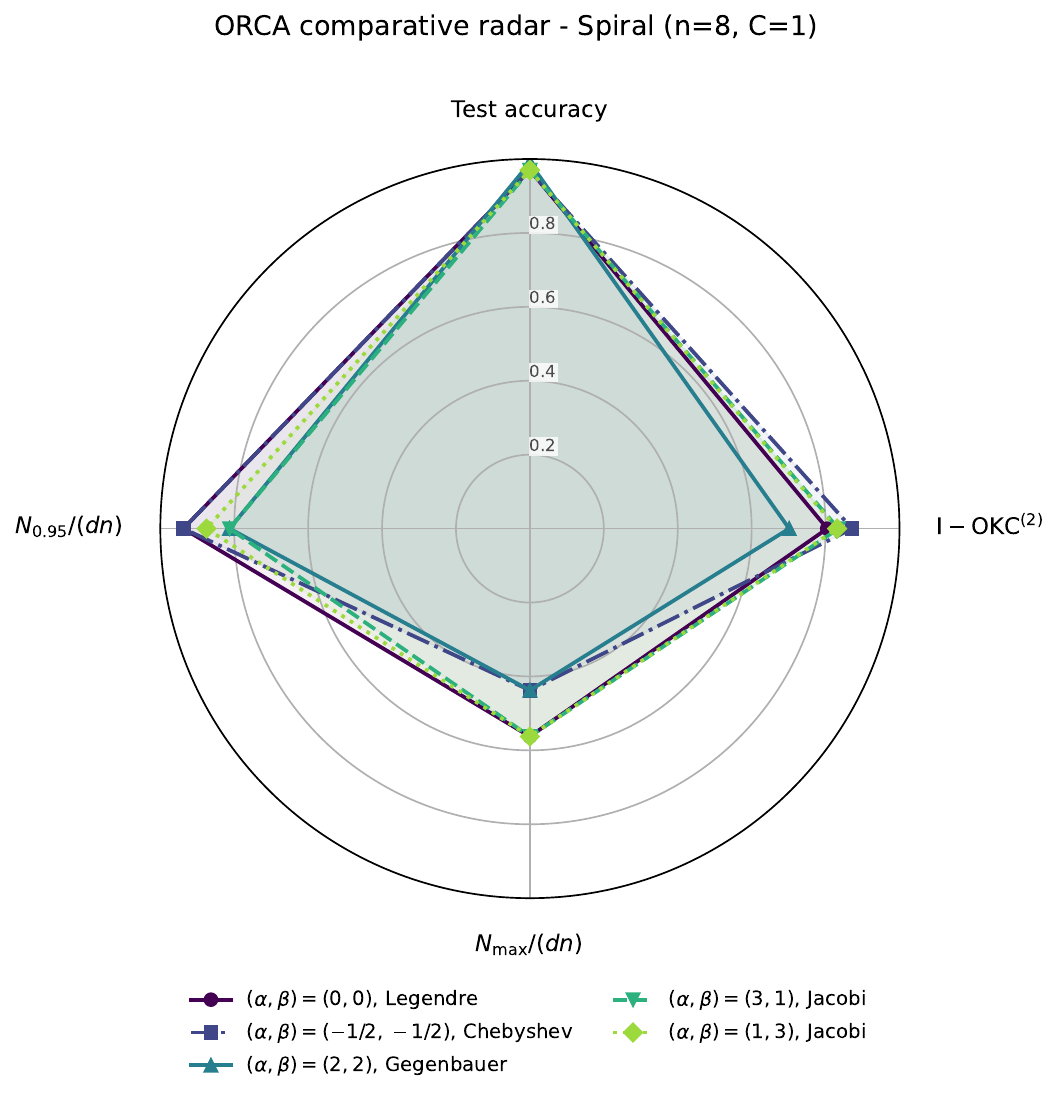}
  \caption{ORCA comparative radar for the five Jacobi tensor-product kernels
  on the double-spiral data set ($d=2$, $n=8$, $C=1$). The axes show the test
  accuracy, $\operatorname{I-OKC}^{(2)}$, $N_{\max}/(dn)$ and $N_{0.95}/(dn)$.}
  \label{fig:spiral_radar}
\end{figure}

The five kernels give similar test accuracies
(Table~\ref{tab:spiral_interaction}): $0.970$ for four of them and $0.990$
for the Gegenbauer kernel, which classifies two more test points correctly.
All five classifiers are dominated by pairwise interactions, with
$\operatorname{I-OKC}^{(2)}$ between $0.70$ and $0.88$, and both input
variables have global indices above $0.80$ and marginal indices below
$0.20$. Both variables are therefore essential and act mainly jointly, as
expected in a problem where the class of a point depends on its distance to
the center together with its angular position.

Among the three symmetric kernels, the marginal part grows with
$\alpha=\beta$: $\operatorname{I-OKC}^{(1)}=0.1281$ for the Chebyshev kernel,
$0.1956$ for the Legendre kernel and $0.2985$ for the Gegenbauer kernel. For
the Gegenbauer kernel, $0.2788$ of this marginal part lies in the cell
$(N,q)=(7,1)$ of its Structural Map, which makes $N=7$ its dominant degree,
with $\operatorname{D-OKC}_7=0.3172$. Its degree profile is also the most
concentrated of the three, with $N_{0.95}=13$ against $15$, whereas the
Legendre and Chebyshev profiles have three degrees with similar
contributions, $N=7,9,11$ and $N=7,11,15$, respectively.

Table~\ref{tab:spiral_degree} also separates the kernels by parity. For the
three symmetric kernels, more than $98\%$ of the squared RKHS norm lies at odd
total degrees, whereas for the two asymmetric kernels about $39\%$ lies at
even degrees. The double spiral is
approximately invariant under a rotation of $180$ degrees combined with the
exchange of the classes. For a symmetric kernel, this makes $h_n^{(d)}$ close
to an odd function, and since its basis polynomials are even or odd according
to their total degree, the squared RKHS norm concentrates on odd degrees. The
basis polynomials of an asymmetric kernel have no definite parity, and this
separation does not occur.

The rows of $(3,1)$ and $(1,3)$ are very close in both tables.
Interchanging $\alpha$ and $\beta$ is equivalent to reflecting the inputs
through the origin, and the double spiral is approximately invariant under
this reflection up to the exchange of the classes, which does not change the
OKC indices.

The comparative radar in Figure~\ref{fig:spiral_radar} summarizes these
results. The five polygons nearly coincide along the accuracy axis. On each
of the three structural axes, the Gegenbauer kernel takes the smallest value,
shared with the Chebyshev kernel for $N_{\max}$ and with $(3,1)$ for
$N_{0.95}$. The Chebyshev kernel has the largest
$\operatorname{I-OKC}^{(2)}$, and the two asymmetric polygons nearly coincide
except on the $N_{0.95}$ axis, where they differ by one degree.

Thus, for the same data and similar accuracies, ORCA distinguishes the
kernels by the share of pairwise interactions, the concentration of the
degree profile and the parity of the total degree.

\section{Real-data application}
\label{sec:realdata}

In this section, ORCA is applied to the Vertebral Column data
set~\cite{uci-vertebral}, taken from the UCI Machine Learning Repository.
Unlike in Section~\ref{sec:spiral}, the SVM specification is not fixed in
advance, since $C$, $n$, $\alpha$ and $\beta$ are selected jointly by
cross-validation on the training sample, and ORCA is applied to the selected
classifier.

The data set contains $310$ patients, each described by $d=6$ biomechanical
attributes derived from the shape and orientation of the pelvis and the
lumbar spine: the pelvic incidence (PI), the pelvic tilt (PT), the lumbar
lordosis angle (LLA), the sacral slope (SS), the pelvic radius (PR) and the
grade of spondylolisthesis (GS). The response distinguishes the $100$
patients classified as normal, $y_i=+1$, from the $210$ classified as
abnormal, $y_i=-1$, of whom $60$ have disk hernia and $150$ have
spondylolisthesis. 
The observations are split by stratified random sampling
with a
random seed of zero
into a training sample of $m=207$ patients, $140$ abnormal and $67$ normal,
and a test sample of $103$ patients, $70$ abnormal and $33$ normal. The inputs
are rescaled with \eqref{eq:scaling} using the training sample, and the test
points are transformed with the same map.

The parameters are selected over the grid
\begin{gather*}
  C\in\{0.01,0.1,1,10,100,1000\},\qquad n\in\{2,\dots,6\},\\
  \alpha,\beta\in\{-0.6,-0.4,-0.2,0,0.2,0.4,0.6,0.8,1,1.5,2,3,4\},
\end{gather*}
which contains $5070$ configurations. The value $n=1$ is excluded, since then
$N(\boldsymbol{k})=q(\boldsymbol{k})$ for every multi-index and the
Structural Map reduces to its diagonal. Each configuration is evaluated by
stratified five-fold cross-validation on the training sample, repeated $20$
times with random seed $123$, and the selection criterion is the mean
validation accuracy over the $100$ folds. In each fold, the rescaling \eqref{eq:scaling} is computed from the
training folds and applied to the validation fold, ensuring that the test
sample remains independent of the selection process. During cross-validation,
a numerical tolerance of $10^{-3}$ and a maximum of $2\times10^{5}$
iterations are used, with non-converged fits excluded from the selection.
The selected classifier is then refitted on the entire training sample
with a tolerance of $10^{-9}$.

The selected configuration is $n^{\star}=2$, $C^{\star}=0.01$ and $(\alpha^{\star},\beta^{\star})=(1.5,1)$,
with mean cross-validated accuracy $0.8447$. It gives
$\dim(\mathcal{H}^{(6)}_2)=729$ and $dn^{\star}=12$. The test accuracy of the selected
classifier is $0.835$, against a majority-class rate of $0.680$ in the test
sample, and its balanced accuracy, the mean of the proportions of correctly
classified test points in each class, is $0.806$. The results are shown in
Table~\ref{tab:vert} and in
Figures~\ref{fig:vert_structural_map} and~\ref{fig:vert_variable_panel}.

\begin{table}[htbp]
\centering
\caption{ORCA indices of the selected SVM for the vertebral column data set
($d=6$, $n^{\star}=2$, $C^{\star}=0.01$, $(\alpha^{\star},\beta^{\star})=(1.5,1)$).}
\label{tab:vert}
\begin{tabular}{lcccccc}
\toprule
Interaction order $q$ & $1$ & $2$ & $3$ & $4$ & $5$ & $6$ \\
\midrule
$\operatorname{I-OKC}^{(q)}$ & 0.0621 & 0.1344 & 0.2923 & 0.2516 & 0.2035 & 0.0561 \\
\midrule
Variable & PI & PT & LLA & SS & PR & GS \\
\midrule
$\operatorname{M-OKC}_j$ & 0.0005 & 0.0014 & 0.0005 & 0.0003 & 0.0017 & 0.0577 \\
$\operatorname{G-OKC}_j$ & 0.5577 & 0.4621 & 0.5119 & 0.5258 & 0.6071 & 0.9038 \\
\bottomrule
\end{tabular}
\end{table}

\begin{figure}[htbp]
  \centering
  \includegraphics[width=0.85\textwidth]{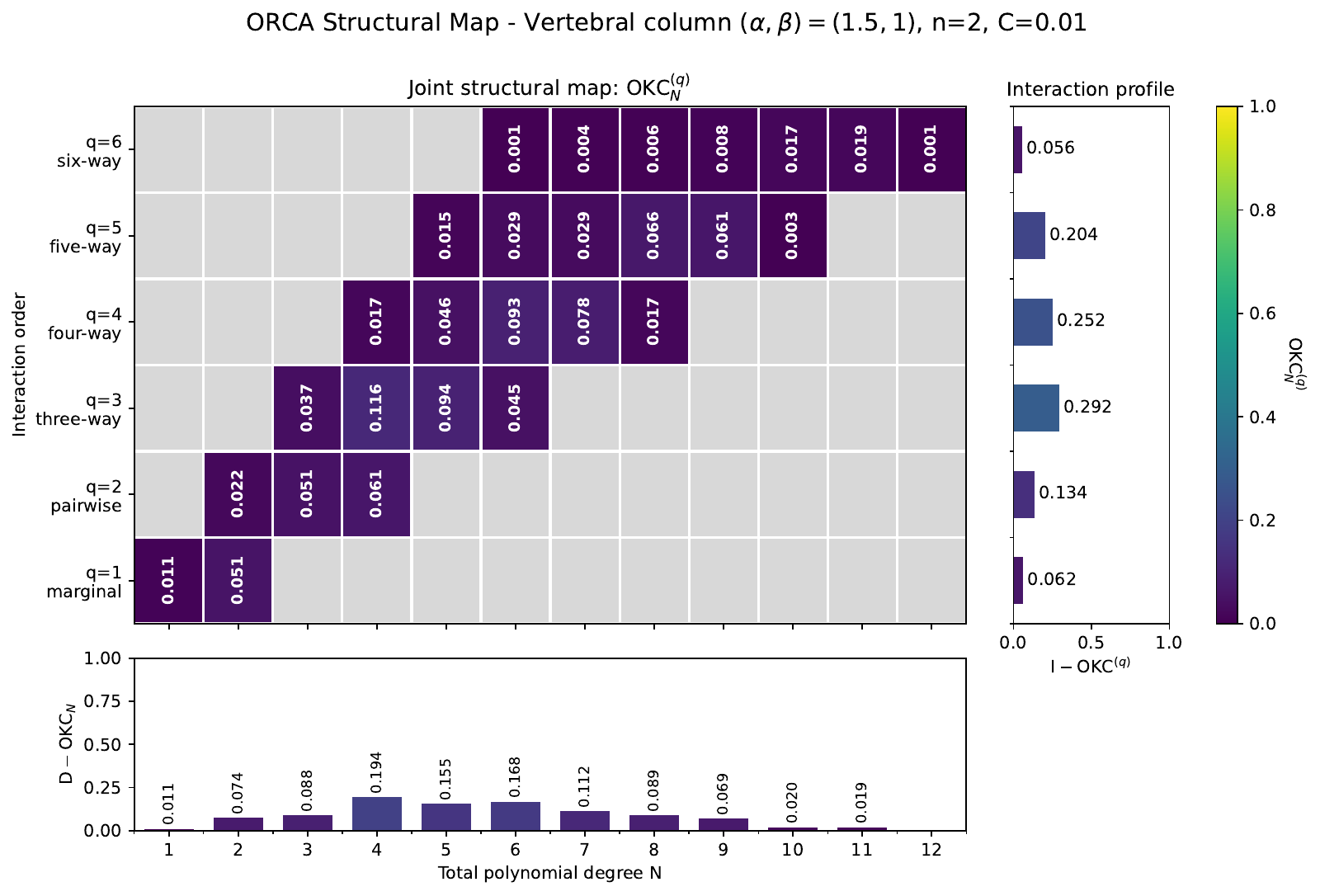}
  \caption{ORCA Structural Map of the selected SVM for the vertebral column
  data set ($d=6$, $n^{\star}=2$, $C^{\star}=0.01$, $(\alpha^{\star},\beta^{\star})=(1.5,1)$).}
  \label{fig:vert_structural_map}
\end{figure}

\begin{figure}[htbp]
  \centering
  \includegraphics[width=0.85\textwidth]{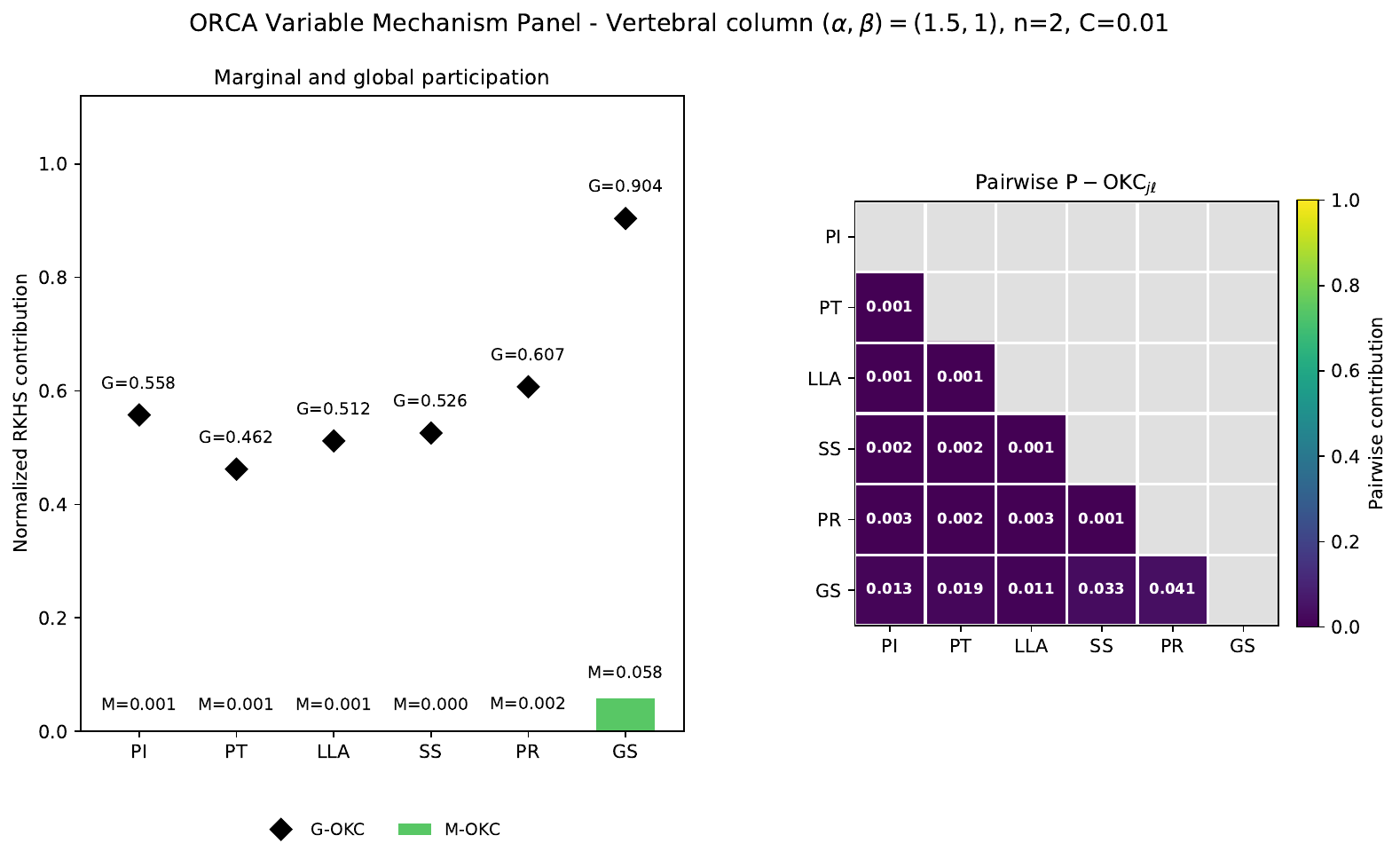}
  \caption{ORCA Variable Mechanism Panel of the selected SVM for the vertebral
  column data set ($d=6$, $n^{\star}=2$, $C^{\star}=0.01$, $(\alpha^{\star},\beta^{\star})=(1.5,1)$).}
  \label{fig:vert_variable_panel}
\end{figure}

The ORCA report is dominated by interactions of order three or more, which
carry $0.80$ of the squared RKHS norm, against $0.13$ for pairwise
interactions and $0.06$ for marginal terms, and $\bar q=3.57$. In the
Structural Map, the largest cell is $(N,q)=(4,3)$, with $0.1163$, the degree
profile peaks at $N_{\max}=4$, and $N_{0.95}=9$ out of $dn^{\star}=12$.

The grade of spondylolisthesis is the dominant variable. It has the largest
global index, $\operatorname{G-OKC}=0.9038$, it carries almost all the
marginal part, with $\operatorname{M-OKC}=0.0577$ out of
$\operatorname{I-OKC}^{(1)}=0.0621$, and it takes part in the five largest
pairwise indices, led by the pair (PR,\,GS) with $0.0414$. The other five
variables have global indices between $0.46$ and $0.61$ and marginal indices
below $0.002$, and more than $90\%$ of each of their global indices comes from
interactions of order three or more. The dominant role of GS agrees with the
composition of the abnormal class, in which $150$ of the $210$ patients have
spondylolisthesis.
\section{Conclusions and future work}
\label{sec:conclusions}

We have studied support vector machine classifiers constructed from truncated tensor-product Christoffel--Darboux kernels. The associated finite orthonormal polynomial basis provides an exact representation of the regularized component of the fitted decision function and an additive decomposition of its squared RKHS norm. Building on this decomposition, the proposed OKC indices provide a quantitative characterization of the classifier's internal structure in terms of total polynomial degree, interaction order, and the contributions associated with explanatory variables. The ORCA methodology brings these indices together into a unified framework for structural analysis, providing joint structural maps, variable-specific summaries, and systematic comparisons of fitted classifiers. An important feature of this approach is its ability to reveal structural differences between classifiers with similar predictive performance, thereby providing information that conventional performance measures alone cannot capture. The controlled classification problems and the Vertebral Column application illustrate the usefulness of ORCA and the OKC indices for understanding and comparing the structure of fitted classifiers.

Future work will focus on two related directions. First, the evolution of the OKC profiles across truncation levels and regularization parameters will be examined alongside validation performance to investigate possible structural indicators of overfitting. Second, the stability of these profiles across different training samples and bootstrap resamples will be studied to assess the robustness of the identified structural patterns. These extensions will further develop the role of ORCA in assessing the structural properties of fitted classifiers.


\end{document}